\documentclass[final,5p,times,twocolumn]{elsarticle}
\usepackage{times}
\usepackage{adjustbox}
\usepackage{tabularx}
\usepackage{amsmath}
\usepackage{multirow}
\usepackage{hyperref}
\usepackage{subcaption}
\usepackage{comment}
\usepackage{amsthm}
\usepackage{blindtext}
\usepackage{float}
\usepackage{enumitem}
\usepackage{multirow}
\usepackage{graphicx}
\usepackage{algorithm}% http://ctan.org/pkg/algorithms
\usepackage{algorithmic}
\usepackage{amsfonts}
\usepackage{amsthm}
\newtheorem{assumption}{Assumption}
\newtheorem{definition}{Definition}
\newtheorem{theorem}{Theorem}

\newtheorem{lemma}{Lemma}
\newtheorem{remark}{Remark}
\usepackage{color}
\usepackage{amssymb}
\usepackage{amsthm}
\setcitestyle{square}
\journal{Neurocomputing}

\begin{document}

\begin{frontmatter}

\title{Estimating Stochastic Linear Combination of Non-linear \\ Regressions Efficiently and Scalably\tnoteref{t2}}
\tnotetext[t2]{This research was supported in part by the National Science Foundation (NSF) through grants  IIS-1910492 and CCF-1716400.}
\author[1]{Di Wang \corref{cor1}}
\ead{dwang45@buffalo.edu}
\author[1]{Xiangyu Guo \corref{cor2}}
\ead{xiangyug@buffalo.edu}
\author[1]{Chaowen Guan}
\ead{chaowen@buffalo.edu} 
\author[1]{Shi Li}
\ead{shil@buffalo.edu}
\author[1]{Jinhui Xu}
\ead{jinhui@buffalo.edu}
\address[1]{Department of Computer Science and Engineering\\
State University of New York at Buffalo\\
338 Davis Hall, Buffalo, 14260}
\cortext[cor1]{Corresponding author}
\cortext[cor2]{The first two authors contributed equally}


\begin{abstract}
{
 Recently, many machine learning and statistical models such as non-linear regressions,  the Single Index, Multi-index, Varying Coefficient Index Models and Two-layer Neural Networks can be reduced to or be seen as a special case of a new model which is called the \textit{Stochastic Linear Combination of Non-linear Regressions} model. However, due to the high non-convexity of the problem, there is no previous work study how to estimate the model. In this paper, we provide the first study on how to estimate the model efficiently and scalably. Specifically, we first show that with some mild assumptions, if the variate vector $x$ is multivariate Gaussian, then there is an algorithm whose output vectors have $\ell_2$-norm estimation errors of $O(\sqrt{\frac{p}{n}})$ with high probability, where $p$ is the dimension of $x$ and $n$ is the number of samples. The key idea of the proof is based on an observation motived by the Stein's lemma. Then we extend our result to the case where $x$ is bounded and sub-Gaussian using the zero-bias transformation, which could be seen as a generalization of the classic Stein's lemma. We also show that with some additional assumptions there is an algorithm whose output vectors have $\ell_\infty$-norm estimation errors of $O(\frac{1}{\sqrt{p}}+\sqrt{\frac{p}{n}})$ with high probability. We also provide a concrete example to show that there exists some link function which satisfies the previous assumptions. Finally, for both Gaussian and sub-Gaussian cases we propose a faster sub-sampling based algorithm and show that when the sub-sample sizes are large enough then the estimation errors will not be sacrificed by too much. Experiments for both cases support our theoretical results. 
    To the best of our knowledge, this is the first work that studies and provides theoretical guarantees for the stochastic linear combination of non-linear regressions model. 
}
\end{abstract}


\end{frontmatter}


\section{Introduction}
{Recently, many machine learning and statistical models  can be reduced to or be seen as a special case of a new model which is called the \textit{Stochastic Linear Combination of Non-linear Regressions} model, which can be defined as the followings. }

\begin{definition}[Stochastic Linear Combination of Non-linear Regressions]\label{def:1}
Given variates $x\in \mathbb{R}^p$ and $z_1, \cdots, z_k\in \mathbb{R}$ such that $\mathbb{E}[x]=0$ and  $z_i$'s for all $i\in[k]$ are i.i.d random variables independent of $x$ with $\mathbb{E}[z_i]=0$ and $\text{Var}(z_i)=1$, the response $y$ is given by
\begin{equation}\label{eq:1}
    y= \sum_{i=1}^k z_i f_i(\langle \beta_i^*, x\rangle)+ \epsilon, 
\end{equation}
where $\beta_1^*, \beta_2^*, \cdots, \beta^*_k\in \mathbb{R}^p$ are unknown parameters, $f_i$'s for all $i\in[k]$ are known (but could be non-convex) \textit{link functions}, and $\epsilon$ is some random noise (from an unknown distribution) satisfying $\mathbb{E}[\epsilon]=0$ and is independent of $x$ and $z_i$'s.

The goal is to estimate the parameters $\beta_j^*$ for all $j\in [k]$ from $n$ observations $(x_1, y_1, \{z_{1,i}\}_{i=1}^k)$, $(x_2, y_2, \{z_{2,i}\}_{i=1}^k), \cdots, (x_n, y_n, \{z_{n,i}\}_{i=1}^k)$. 
\end{definition}

This model has a close connection with many models in Statistics, Machine Learning, Signal Processing and Information Theory: (1) when $k=1$, the model is reduced to the non-linear regression estimation problem which has been studied in  \cite{zhang2018nonlinear,beck2013sparse,yang2016sparse,cook1999dimension} and is related to compressed sensing and image recovery as well; (2) when $k=1$ but the link function $f_1$ is unknown, it becomes the \textit{ \bf Single Index Model}, which is one of the most fundamental models in statistics and has been studied for many years \cite{ichimura1993semiparametric,horowitz2009semiparametric,kakade2011efficient,yang2017high,radchenko2015high}; (3) when $k\geq 1$, $z_i$'s are deterministic but $f_i$'s are unknown, this model will be a special case of the \textit{\bf Multi-index Model} which has been studied in \cite{li1991sliced,li1992principal,li1989regression,yang2017learning}; (4) 
when $k\geq 1$, $z_i$'s are stochastic but $f_i$'s are unknown, it will be the \textit{\bf Varying Coefficient Index Model} which was introduced by \cite{ma2015varying} and has wide applications in economics and medical science \cite{fan2008statistical}; (5) when all $f_i$'s are the same, the model can be viewed as a {\bf Two-layer Neural Network} with $k$ hidden nodes and random hidden-output layer weights.

{To estimate the parameters in Model (\ref{def:1}), the main challenge is that without the assumption that $f_i$'s are convex or similarities between them, it is hard to establish an objective function that can be efficiently optimized using optimization methods such as (Stochastic) Gradient Descent. Thus, due to the high convexity and randomness, there is no previous study how to solve the model efficiently. Recently, some works including \cite{yang2017high,na2018high,yang2017learning} studied and proposed efficient algorithms for the Single Index, Multi-index and Varying Coefficient Index models using Stein's Lemma. Their theoretical guarantees are measured in terms of $\|\beta_j-c\beta_j^*\|_2, j\in[k]$, where $\beta_j$ is the estimator for $\beta_j^*$ and $c$ is a constant depending on many parameters in the models (such as $f_i$'s, $\beta^*_j$'s and the distribution for $x$).  However there is a common issue related to the constant $c$ in these results: They did not provide a method to compute or even estimate $c$. %we prefer the measurement of the difference to $\beta^*$, 
Moreover, measuring the error by using the terms $\|\beta_j-c\beta_j^*\|_2, j\in[k]$ is quite meaningless due to the constant $c$. Thus, we wish to use other measurements and ideally, we hope to measure the error in terms of $\beta_j- \beta_j^*$ for all $j\in [k]$; that is, we do not introduce the constant $c$.}
%With the aforementioned issue in mind, we want to ask: 
The key question the paper tries to answer is:

{\bf Is there an efficient method whose output vectors $\beta_1$, $\beta_2, \cdots, \beta_k$ have small errors compared to $\beta^*_1, \beta^*_2, \cdots, \beta^*_k$?}

In this paper, we answer the question in the affirmative under some mild assumptions on the model. Specifically, our contributions can be summarized as follows.
\begin{enumerate}
    \item We first consider the case where $x$ multivariate Gaussian. In this case, we show that there is a special structure for each $\beta_j^*$, $j\in [k]$: $\beta^*_j= c_j \beta_j^{ols}$, where $c_j$ is a constant depending on the link function $f_j$ and $x$, and $\beta_j^{ols}$ is the Ordinary Lest Square estimator w.r.t $yz_j$ and $x$, {\em i.e.,} $\beta_j^{ols}=(\mathbb{E}[xx^T])^{-1}\mathbb{E}[z_jyx]$. Based on this key observation, we propose an algorithm which estimates $c_j$'s and $\beta_j^{ols}$'s, and outputs $\{\beta_j\}_{j=1}^{k}$ satisfying $\|\beta_j-\beta_j^*\|_2\leq O(\sqrt{\frac{p}{n}})$ for each $j\in [k]$ with high probability. Moreover, in order to make our algorithm faster, instead of using linear regression estimator to approximate $\beta_j^{ols}$, we use the sub-sampling covariance linear regression estimator \cite{dhillon2013new}. We show that if the sub-sample size is large enough, the error bound is almost the same as in the previous ones. 
    \item We then extend our result to the case when $x$ is (bounded)  sub-Gaussian. The challenge is that the result for the Gaussian case depends on some properties of Gaussian distribution which are not satisfied in the sub-Gaussian case. To overcome this, we use the zero-bias transformation \cite{goldstein1997stein}, which could be seen as a generalization of the Stein's lemma \cite{brillinger1982generalized}. Particularly, we show that instead of the equality $\beta^*_j= c_j \beta_j^{ols}$, we have the $\ell_\infty$ norm estimation error $\|\beta^*_j-  c_j \beta_j^{ols}\|_{\infty}\leq O(\frac{1}{\sqrt{p}})$ with some additional mild assumptions. Based on this and the same idea from the Gaussian case, we show that there exists an algorithm whose output vectors $\{\beta_j\}_{j=1}^{k}$ satisfy $\|\beta_j-\beta_j^*\|_\infty \leq O(\frac{1}{\sqrt{p}}+\sqrt{\frac{p}{n}})$ with high probability. Similarly, we also propose a sub-sampled version of our algorithm as in the Gaussian case.
    \item {While we provide some theoretical results in the previous parts, it is still unknown whether there exists any link function which satisfies these assumptions. To solve this problem, we consider the case where the link functions are sigmod function. And we show that with some assumptions on $x$, it indeed satisfies the previous assumptions.}
    \item {At the end, we show the experimental results on both Gaussian and sub-Gaussian cases with single/mixed type of link functions, and these results support our theoretical results above.  Specially, they show both of the effectiveness and scalability of our previous algorithms.}
\end{enumerate}
To the best of our knowledge, this is the first paper studying and providing the estimation error bound for Model (\ref{def:1}) in both Gaussian and  sub-Gaussian cases.

{This paper is a substantially extended version of our previous work appeared in AAAI'20 \cite{wangaaai20203}. The following are the main added contents. Firstly, we construct a concrete loss function which satisfies the assumptions in our theorems which has been studied in the conference version. Specifically, we show that under some assumptions on $x$, when the link function is the sigmoid function, then it satisfies the assumptions in Theorem 2 (see Theorem 7 for details). Secondly, we provide the proofs for all theorems and lemmas, and we believe the techniques can be used to other problems. }

{The rest of the paper is organized as follows. Section \ref{sec:2} introduces some related work. Section \ref{sec:3} gives some preliminaries on Sub-Gaussian random variable, necessary lemmas and assumptions throughout the paper. Section \ref{sec:4} describes our proposed algorithms for Gaussian case. Section \ref{sec:5} extends our algorithm to the Sub-Gaussian case. We provide all the proofs in Section \ref{sec:6}. Finally, we experimentally study our methods in Section \ref{sec:7}, and conclude them in Section \ref{sec:8}.
}
\section{Related Work}\label{sec:2}
As we  mentioned above, there is no previous work on Model (\ref{def:1}) with guarantees on the $\ell_2$ or $\ell_\infty$ norm of the errors $\beta_j - \beta_j^*$. Hence, below we compare with the results which are close to ours. 

When the link functions $f_j$'s are unknown,  Model (\ref{eq:1}) is just the Varying Coefficient Index Model. \cite{na2018high} provided the first efficient algorithm for this model. Although they considered the high dimensional sparse case, their method requires the underlying distribution of $x$ to be known, an unrealistic assumption for most applications.
%which will be unavailable in the real case. 
Moreover, their estimation errors are measured by the differences between $\beta_j$'s and $c\beta^*_j$'s for an unknown $c$, while in our results we have fixed $c = 1$. %our estimation errors are the differences to $\beta_j^*$.

When the link functions $f_j$'s are all the same, then our model can be reduced to the two-layer neural network with random hidden-output layer weights. Previous work on the convergence results all focused on the gradient descent type of methods such as those in \cite{zhang2019learning,zou2018stochastic,nitanda2019refined}. However, our method is based on Stein's lemma and its generalization. Compared with the gradient descent type methods, our algorithm is non-interactive (that is, we do not need to update estimators in each iteration) and parameter-free (that is, we don not need to tune the step-size). Moreover, our method can be extended to the case where the link functions $f_j$'s are different. 

Our method is motivated by Stein's lemma \cite{brillinger1982generalized} and its generalization, the zero-bias transformation. Several previous studies have used Stein's Lemma in various machine learning problems. For example, \cite{erdogdu2016scaled,erdogdu2016newton} used it to accelerate some optimization procedures,  \cite{liu2016stein} applied it to Bayesian inference, \cite{wang2019estimating} appied it to estimate smooth Generalized Linear Model in differential privacy model and \cite{yang2017high,yang2016sparse,na2018high,wei2019statistical} used it and its generalizations in the Single Index, Multi-index, Varying Coefficient Index and Generative models, respectively. The zero-bias transformation has also been used in \cite{erdogdu2019scalable} for estimating the Generalized Linear Model. However, due to the difference between the models, these algorithms cannot be applied to our problem.  

\section{Preliminaries}\label{sec:3}
In this section, we review some necessary definitions and lemmas. 
\begin{definition}[Sub-Gaussian] \label{def:2}
For a given constant $\kappa$, a random variable $x\in \mathbb{R}$ is said to be sub-Gaussian if it satisfies $\sup_{m\geq 1}\frac{1}{\sqrt{m}}\mathbb{E}[|x|^m]^{\frac{1}{m}}\leq \kappa$. The smallest such $\kappa$ is the sub-Gaussian norm of $x$ and it is denoted by $\|x\|_{\psi_2}$. 

Similarly, a random vector $x\in \mathbb{R}^p$ is called a sub-Gaussian vector if there exists a constant $\kappa$ such that $\sup_{v\in S^{p-1}}\|\langle x, v \rangle \|_{\psi_2}\leq \kappa$, where $S^{p-1}$ is the set of all $p$-dimensional unit vector. 
\end{definition}
In order to extend our results to the sub-Gaussian case, we will use the zero-bias transformation which is proposed by \cite{goldstein1997stein}. It is a generalization of the classic Stein's lemma in \cite{brillinger1982generalized}.
\begin{definition}\label{def:3}
Let $z$ be a random variable with mean $0$ and variance $\sigma^2$. Then there exists a random variable $z^*$ such that for all differentiable functions $f$ we have $\mathbb{E}[zf(z)]= \sigma^2\mathbb{E}[f'(z^*)]$. The distribution of $z^*$ is said to be the $z$-zero-bias distribution. 
\end{definition}
The standard Gaussian distribution is the unique distribution whose zero-bias distribution is itself. This is just the basic {\bf Stein's lemma}. 

\begin{lemma}\label{alemma:1}\cite{dhillon2013new}
Assume that $\mathbb{E}[x]=0, \mathbb{E}[x_ix_i^T]=\Sigma\in \mathbb{R}^{p\times p}$, and $\Sigma^{-\frac{1}{2}}x$ and $y$ are sub-Gaussian with norms $\kappa_x$ and $\gamma$ respectively. If $n\geq \Omega(\gamma\kappa_x p)$, then with probability at least  $1-3\exp(-p)$ we have 
\begin{equation}
    \|\Sigma^{\frac{1}{2}}(\tilde{\beta}^{ols}- \beta^{ols})\|_2 \leq C_1 \kappa_x \gamma \sqrt{\frac{p}{n}}, 
\end{equation}
where $\beta^{ols}= \Sigma^{-1}\mathbb{E}[yx]$ is the OLS estimator w.r.t $y$ and $x$,  $\tilde{\beta}^{ols}= (X^TX)^{-1}X^TY$ is the empirical one, and $C_1>0$ is some universal constant. 
\end{lemma}

\begin{lemma}\label{alemma:3}\cite{erdogdu2019scalable}
Let $B^\delta(\tilde{\beta})$ denote the ball centered around $\tilde{\beta}$ with radius $\delta$. For $i=1, 2, \cdots, n$, let $x_i\in\mathbb{R}^p$ be i.i.d random vectors with a covariance matrix $\Sigma$. Given a function $g$ that is uniformly bounded by $L$ and $G$-Lipschitz, with probability at least $1-\exp(-p)$ we have
\begin{equation*}
      \sup_{\beta\in B^\delta(\tilde{\beta})}\Bigg|\frac{1}{n}\sum_{i=1}^n g(\langle x_i, \beta\rangle)- \mathbb{E}[g(\langle x, \beta \rangle)]\Bigg|\leq 
      2(G(\|\tilde{\beta}\|_2+\delta)\|\Sigma\|_2 + L)\sqrt{\frac{p}{n}}. 
\end{equation*}
\end{lemma}
\begin{assumption}\label{ass:1}
We assume that for each $j\in [k]$, the random variable $yz_j$ is sub-Gaussian with its sub-Gaussian norm $\|yz_j\|_{\psi_2}=\gamma$.
\end{assumption}
Note that this assumption holds if $y$ is bounded and $z_j$ is sub-Gaussian or $z_j$ is bounded and $y$ is sub-Gaussian.  
\begin{assumption}\label{ass:2}
We assume that there exist  constants $G, L>0$ such that for each $j\in [k]$, $f_j'$ is $G$-Lipschitz and bounded by $L$. Also for $j\in [k]$, we let $\mathbb{E}[f'_j(\langle x,\beta^*_j\rangle )]\neq 0$. 
\end{assumption}
\paragraph{Notations} For a positive semi-definite matrix $M\in \mathbb{R}^{p\times p}$, we define the $M$-norm for a vector $w$ as $\|w\|_M^2= w^TMw$. Also we will denote $B^{\delta}_M(\tilde{\beta})$ as the ball around $\tilde{\beta}$ with radius $\delta$ under $M$-norm, {\em i.e.,} $B^{\delta}_M(\tilde{\beta})=\{\beta: \|M^{\frac{1}{2}}(\beta- \tilde{\beta})\|_2\leq \delta \}$.
$\lambda_{\min}(M)$ is the minimal singular value of the matrix $M$. For a semi positive definite matrix $M\in\mathbb{R}^{p\times p}$, let its SVD be  $M=U^T\Sigma U$, where $\Sigma=\text{diag}(\lambda_1, \cdots, \lambda_p)$, then $M^{\frac{1}{2}}$ is defined as $M^{\frac{1}{2}}=U^T\Sigma^{\frac{1}{2}}U$ with $\Sigma^{\frac{1}{2}}=\text{diag}(\sqrt{\lambda_1}, \cdots, \sqrt{\lambda_p})$.

\section{Gaussian Case}\label{sec:4}
In this section we consider the case where $x$ is sampled from some multivariate Gaussian distribution, then we will extend our idea to the sub-Gaussian distribution case in next section.

Our algorithm is based on the following key observation using some properties of the multivariate Gaussian distribution. 
\begin{theorem}\label{thm:1}
Consider Model (\ref{def:1}) in Definition \ref{def:1} under Assumptions \ref{ass:1} and \ref{ass:2}. Moreover, assume that the observations $\{x_i\}_{i=1}^n$ are i.i.d sampled from $\mathcal{N}(0, \Sigma)$. Then each $\beta^*_j, j\in[k]$ can be written as 
\begin{equation}\label{eq:2}
    \beta^*_j= c_j\times  \beta_j^{ols},
\end{equation}
where $\beta_j^{ols}=\Sigma^{-1}\mathbb{E}[z_jyx]$ and  $c_j$ is the root of the function $l_j(c)-1$ where
\begin{equation}\label{eq:3}
    l_j(c)= c\mathbb{E}[f_j'(\langle x, \beta_j^{ols}\rangle c)].
\end{equation}
\end{theorem}
From Theorem \ref{thm:1} we can see that, in order to estimate $\beta^*_j$, it is sufficient to estimate the terms  $\beta_j^{ols}=\Sigma^{-1}\mathbb{E}[z_jyx]$ and $c_j$. If we denote $z_jy$ as the response and $x$ as the variate, then the term $\beta_j^{ols}$ is just the \textbf{Ordinary Least Square (OLS)} estimator. Thus we can use its empirical form $\tilde{\beta_j}^{ols}=(\sum_{i=1}^nx_i^Tx_i)^{-1}\sum_{i=1}^n z_{i,j}y_ix_i= (X^TX)^{-1}X^TY_j$ as an estimator, where $X=[x_1^T; x_2^T; \cdots; x_n^T]\in\mathbb{R}^{n\times d} $ is the data matrix and $Y_j=[z_{1,j}y_1, \cdots, z_{n,j}y_n]^T$ is the corresponding response vector. 

After getting the estimator of $\beta_j^{ols}$, denoted by $\tilde{\beta}_j^{ols}$, we use it to approximate $c_j$. That is we find the root $\hat{c}_j$ of the empirical version of ${l}_j(c)-1$, {\em i.e.,  $\hat{l}_j(c)-1$},  where 
\begin{equation*}
    \hat{l}_j(c)=\frac{c}{n}\sum_{i=1}^n [f_j'(\langle x_i, \tilde{\beta}_j^{ols}\rangle c )].
\end{equation*}
Note that there are numerous methods available to find a root of a function, such as Newton's root-finding method with quadratic convergence and Halley's method with cubic convergence. We also note that this step only cost $O(n)$ per-iteration. After that, we could estimate each $\beta^*_j$ by $\hat{\beta}_j^{nlr}=\hat{c}_j\tilde{\beta}_j^{ols}$. 
In total, we have Algorithm \ref{alg:1}.

\begin{algorithm}[!ht]
\caption{SLS: Scaled Least Squared Estimators}
	$\mathbf{Input}$: Data $\{(x_i, y_i, \{z_{i,j}\}_{j=1}^k)\}_{i=1}^n$, link functions $\{f_j\}_{j\in [k]}$.  
	\label{alg:1}
	\begin{algorithmic}[1]
	\STATE{\bf Option 1:} Let $X= [x_1, x_2, \cdots, x_n]^T\in \mathbb{R}^{n\times p}$ and compute the $\hat{\Sigma}^{-1}=(X^TX)^{-1}$. 
	  \STATE {\bf Option 2:}  Construct a sub-sampling based OLS estimator, that is let $S\subset [n]$ be a random subset and take $\hat{\Sigma}_{S}^{-1}=\frac{|S|}{n}(X^T_SX_S)^{-1}$, where $X_S\in \mathbb{R}^{|S|\times p}$ is the data matrix constrained  on indices of $S$. 
     \FOR{$j=1, 2 \cdots, k$}
     \STATE   Let $Y_i= [z_{1, j}y_1, \cdots, z_{n,j}y_n]^T\in \mathbb{R}^n$.  For {\bf Option 1}, Compute the ordinary least squares estimator $\tilde{\beta}_j^{ols}= (\Sigma)^{-1}X^TY_j$.   For {\bf Option 2}, take $\tilde{\beta}_j^{ols}= (\hat{\Sigma}_S)^{-1}X^TY_j$. 
   
     \STATE Denote $\tilde{y}_j= X \tilde{\beta}_j^{ols}$. Then use the Newton's root-finding method  to the function $\frac{c}{n}\sum_{i=1}^n [f_j'(\tilde{y}_{j, i} c )]-1$, denote the root as $\hat{c}_j$: 
         \FOR {$t=1, 2, \cdots$ until convergence}
     \STATE $c= c- \frac{c\frac{1}{n}\sum_{i=1}^{n}f_j'(c\tilde{y}_{j,i})-1}{
     \frac{1}{n}\sum_{i=1}^{n}\{f_j'(c\tilde{y}_{j,i})+c\tilde{y}_{j, i}f_j^{(2)}(c\tilde{y}_{j,i})\}}$. 
     \ENDFOR
     \STATE $\hat{\beta}_j^{nlr}= \hat{c}_{j}\cdot \tilde{\beta}_j^{ols}$.  
     \ENDFOR
     \STATE  \RETURN $\big(\hat{\beta}_j^{nlr}\big)_{j \in [k]}$
	\end{algorithmic}
\end{algorithm}

The following theorem shows that the converge rate of the estimation error for each  $\|\hat{\beta}_j^{nlr}-\beta^*_j\|_2$ is $O(\sqrt{\frac{p}{n}})$ under some additional mild assumptions on link functions $\{f_j\}_{j=1}^k$. 
\begin{theorem}\label{thm:new}
Consider Option 1 in Algorithm \ref{alg:1}. Under the Assumptions \ref{ass:1}, \ref{ass:2} and the assumptions in Theorem \ref{thm:1}, for each $j\in [k]$ we define the function $\ell_j(c, \beta)= c\mathbb{E}[f_j'(\langle x, \beta\rangle c)]$ and its empirical counter part as 
\begin{equation*}
    \hat{\ell}_j (c, \beta)= \frac{c}{n}\sum_{i=1}^n f_j'(\langle x_i, \beta\rangle c). 
\end{equation*}
Assume that there exist some constants $\eta, \bar{c}_j$ such that $\ell_j(\bar{c}_j, \beta_j^{ols})>1+\eta$. Then there exists $c_j>0$ satisfying the equation $1=\ell_j(c_j, \beta_j^{ols})$ for each $j\in [k]$. 

Further, assume that $n$ is sufficiently large: $$n\geq \Omega({p\|\Sigma\|_2\|\beta_j^*\|_2^2 }) $$ Then, with probability at least $1-k\exp(-p)$ there exist constants $\hat{c}_j\in (0, \bar{c}_j)$ satisfying the equations 
\begin{equation*}
    1= \frac{\hat{c}_j}{n}\sum_{i=1}^n f'_j(\langle x_i, \tilde{\beta}_j^{ols}\rangle \hat{c}_j). 
\end{equation*}
Moreover, if for all $j\in [k]$ the derivative of $z\mapsto \ell_j(z, \beta_j^{ols})$ is bounded below in absolute value ({\em does not change sign}) by $M>0$ in the interval $z\in [0, c_j]$. Then with probability at least $1-4k\exp(-p)$ the outputs $\{\hat{\beta}_j^{nlr}\}_{j=1}^k$  satisfy for each $j\in [k]$
\begin{equation}
           \|\hat{\beta}_j^{nlr}- \beta^*_j\|_2 \leq O(\max\{1, \|\beta_j^*\|_2\}\lambda^{-\frac{1}{2}}_{\min}(\Sigma)\sqrt{\frac{p}{n}} ),
\end{equation}
where $G, L, \gamma,  M, c_j, \eta$ are assumed to be $\Theta(1)$ and thus omitted in the Big-$O$ and $\Omega$ notations (see Appendix for the explicit forms).
\end{theorem}
Note that in Theorem \ref{thm:new} the link functions $f_j$ are not required to be convex. Hence this is quite useful in non-convex learning models.

\paragraph{Time Complexity Analysis} Under Option 1 of Algorithm \ref{alg:1}, we can see that the first step takes $O(np^2+p^3)$ time, calculating $\tilde{\beta}^{ols}_j$ for all $j\in [k]$ takes $O(k(np+ p^2))$ time and each iteration of finding $\hat{c}_j$ takes $O(n)$ time. Thus, if $k$ ,the number of link functions $f_j$, is a constant, then the total time complexity is $O(np^2+p^3+nT)$, where $T$ is the number of iterations for finding $c_j$. 

However, the term $np^2$ is prohibitive in the large scale setting where $n, p$ are huge (see \cite{wang2018large,DBLP:journals/ijon/WangX19} for details). To further reduce the time complexity, we propose another estimator based on sub-sampling. 

Note that the term $O(np^2)$ comes from calculating the empirical covariance matrix $X^TX$. Thus, to reduce the time complexity, instead of calculating the covariance via the whole dataset, we here use the sub-sampled covariance matrix. More precisely, we first randomly sample a set of indices $S\subset [n]$ whose size $|S|$ will be specified later. Then we calculate $\frac{|S|}{n}(X_S^TX_S)^{-1}$ to estimate $(X^TX)^{-1}$, where $X_S\in \mathbb{R}^{|S|\times p}$ is the data matrix constrained  on indices of $S$. We can see that the time complexity in this case will only be $O(|S|p^2+p^3)$. The following lemma, which is given by \cite{dhillon2013new,erdogdu2016scaled} shows the convergence rate of the OLS estimator based on the sub-sampled covariance matrix. This is a generalization of Lemma \ref{alemma:1}. 
\begin{lemma}\label{alemma:new}
Under the same assumptions as in Lemma \ref{alemma:1}, if $|S|\geq \Omega(\gamma \kappa_x p)$, then   with probability at least $1-3\exp(-p)$ the sub-sampled covariance OLS estimator $\tilde{\beta}^{ols}=\frac{|S|}{n}(X_S^TX_S)^{-1}X^TY$ satisfies 
\begin{equation*}
    \|\tilde{\beta}^{ols}-\beta^{ols}\|_2\leq C_2\kappa_x\gamma\sqrt{\frac{p}{|S|}}. 
\end{equation*}
\end{lemma}
We have the following approximation error based the sub-sampled covariance OLS estimator: 
\begin{theorem}\label{thm:new1}
Under the same assumptions as in Theorem \ref{thm:new}, in Algorithm \ref{alg:1} if we use Option 2 with $|S|\geq \Omega(\gamma\kappa_x p)$ , then with probability at least $1-4k\exp(-p)$ the outputs $\{\hat{\beta}_j^{nlr}\}_{j=1}^k$  satisfy for each $j\in [k]$
\begin{equation*}
           \|\hat{\beta}_j^{nlr}- \beta^*_j\|_2 \leq O(\max\{1, \|\beta_j^*\|_2\}\lambda^{-\frac{1}{2}}_{\min}(\Sigma)\sqrt{\frac{p}{|S|}} ). 
\end{equation*}
\end{theorem}
Moreover, it is also possible to accelerate the algorithm using the sub-sampling method in the step 5 (finding the root) and we can see the estimation error will be the same as in Theorem \ref{thm:new1} (by the proof of Theorem \ref{thm:new1}). Due to the space limit, we omit it here. 

\section{Extension to Sub-Gaussian Case}\label{sec:5}
Note that Theorem \ref{thm:new} is only suitable for the case when $x$ is Gaussian. This is due to the requirements on some properties of the Gaussian distribution in the proof of Theorem \ref{thm:1}. In this section we will first extend Theorem \ref{thm:1} to the sub-Gaussian case. 

Remember that the proof of Theorem \ref{thm:1} is based on the classic Stein's lemma. Thus, in order to generalize to sub-Gaussian case, we will use the zero-bias transformation in Definition \ref{def:3} since it is a generalization of the Stein's lemma. With some additional assumptions, we can get the following theorem.

\begin{theorem}\label{thm:2}
Let $x_1, \cdots, x_n\in \mathbb{R}^p$ be i.i.d realizations of a random vector $x$ which is sub-Gaussian with zero mean, whose covariance matrix $\Sigma$ has $\Sigma^{\frac{1}{2}}$ being diagonally dominant \footnote{A square matrix is said to be diagonally dominant if, for every row of the matrix, the magnitude of the diagonal entry in a row is larger than or equal to the sum of the magnitudes of all the other (non-diagonal) entries in that row.}, and whose distribution is supported on a $\ell_2$-norm ball of radius $r$. Let $v=\Sigma^{-\frac{1}{2}}x$ be the whitened random vector of $x$ with 
%Furthermore, denote the 
sub-Gaussian norm 
%of the whitened random vector %$v=\Sigma^{-\frac{1}{2}}x$ as 
$\|v\|_{\psi_2}=\kappa_x$. If   for all $j\in [k]$,
%we assume 
each  $v$ has constant first and second conditional moments ({\em i.e.,} $\forall s\in[p]$ and $\tilde{\beta}_j=\Sigma^{\frac{1}{2}}\beta_j^*$,  $\mathbb{E}[v_{s}|\sum_{t\neq s}\tilde{\beta}_jv_{t}]$ and $\mathbb{E}[v^2_{s}|\sum_{t\neq s}\tilde{\beta}_jv_{t}]$ are deterministic) and the link functions $f_j'$ satisfy Assumption \ref{ass:2}. Then for $c_{j}=\frac{1}{\mathbb{E}[f_j' (\langle x, \beta_j^*\rangle)] }$, the following holds for the model in (\ref{eq:1}) for all $j\in [k]$
\begin{equation}\label{eq:4}
    \|\frac{1}{c_j}\cdot \beta_j^*-\beta_j^{ols}\|_{\infty}\leq 16Gr\kappa_x^3\sqrt{\rho_2}\rho_{\infty}\frac{\|\beta_j^*\|^2_\infty}{\sqrt{p}},  
\end{equation}
where $\rho_q$ for $q=\{2, \infty\}$ is the conditional number of $\Sigma$ in $\ell_q$ norm and $\beta_j^{ols}=\Sigma^{-1}\mathbb{E}[xyz_j]$ is the OLS vector w.r.t $yz_j$ and $x$.  
\end{theorem}

\begin{remark}
Note that compared with the equality relationship between $\beta_j^*$ and $c_j\beta^{ols}$ in Theorem \ref{thm:1}, in Theorem \ref{thm:2} we only has the $\ell_\infty$ norm of their difference. Also, here we need more assumptions on the distribution of $x$, and these assumptions ensure that the estimation error decays in the rate of $O(\frac{1}{\sqrt{p}})$. 
\end{remark}
Theorem \ref{thm:2} indicates that we can use the same idea as in the Gaussian case to estimate each $\beta_j^*$. Note that the forms of $c_j$ in Theorem  \ref{thm:1} and \ref{thm:2} are different. In Theorem \ref{thm:1} each $c_j$ is based on $\beta_j^{ols}$, while in Theorem \ref{thm:2} it is based on $\beta_j^*$.  However, due to the closeness of $\beta_j^*$ and $\beta_j^{ols}$ in (\ref{eq:4}), we can still use $\frac{1}{\mathbb{E}[f'_j(\langle x_i, \beta_j^{ols}\rangle \tilde{c}_j)] }$ to approximate $c_j$, where $\tilde{c}_j$ is the root of $c\mathbb{E}[f'_j(\langle x_i, \beta_j^{ols}\rangle c)]-1$. Because of this similarity, we can still use Algorithm \ref{alg:1} for the sub-Gaussian case under the assumptions in Theorem \ref{thm:2}, and we can get the following estimation error:
\begin{theorem}\label{thm:3}
Consider Option 1 in Algorithm \ref{alg:1}. Under Assumptions \ref{ass:1}, \ref{ass:2} and the assumptions in Theorem \ref{thm:2}, for each $j\in [k]$, if we define the function $\ell_j(c, \beta)= c\mathbb{E}[f_j'(\langle x, \beta\rangle c)]$ and its empirical counter part as 
\begin{equation*}
    \hat{\ell}_j (c, \beta)= \frac{c}{n}\sum_{i=1}^n f_j'(\langle x_i, \beta\rangle c). 
\end{equation*}
Assume that there exist some constants $\eta, \bar{c}_j$ such that $\ell_j(\bar{c}_j, \beta_j^{ols})>1+\eta$. Then there exists $\tilde{c}_j>0$ satisfying the equation $1=\ell_j(\tilde{c}_j, \beta_j^{ols})$ for each $j\in [k]$. 

Further, assume that $n$ is sufficiently large:
  $$  n\geq \Omega(\|\Sigma\|_2{p^2\rho_2\rho^2_{\infty} \|\beta^*_j\|^2_{\infty}\max\{1, \|\beta^*_j\|_\infty^2\}}).$$
Then, with probability at least $1-k\exp(-p)$ there exist constants $\hat{c}_j\in (0, \bar{c}_j)$ satisfying the equations 
\begin{equation*}
    1= \frac{\hat{c}_j}{n}\sum_{i=1}^n f'_j(\langle x_i, \tilde{\beta}_j^{ols}\rangle \hat{c}_j). 
\end{equation*}
Moreover, if for all $j\in [k]$, the derivative of $z\mapsto \ell_j(z, \beta_j^{ols})$ is bounded below in absolute value ({\em does not change sign}) by $M>0$ in the interval $z\in [0, \max\{\bar{c}_j, c_j\}]$. Then  with probability at least $1-4k\exp(-p)$ the outputs $\{\hat{\beta}_j^{nlr}\}_{j=1}^k$  satisfy for each $j\in [k]$
\begin{multline*}
        \|\hat{\beta}_j^{nlr}- \beta^*_j\|_\infty \leq O\big(\sqrt{\rho_2}\rho_\infty \lambda^{-\frac{1}{2}}_{\min} (\Sigma) \sqrt{\frac{p}{n}}\|\beta_j^*\|_\infty \\ \times \max\{1,\|\beta_j^*\|_\infty\}  +\rho_2\rho_{\infty}^2 \frac{\max\{\|\beta_j^*\|_{\infty}^2, 1\}\|\beta_j^*\|_{\infty}^2}{\sqrt{p}}\big), 
\end{multline*}
where $\eta, G, L, \gamma, M, \bar{c}_j, r, \kappa_x, c_j$ are assumed to be $\Theta(1)$ and thus omitted in the Big-$O$ and $\Omega$ notations (see Appendix for the explicit forms).
\end{theorem}

\begin{remark}
Compared with the converge rate in the $\ell_2$-norm  of $O(\sqrt{\frac{p}{n}})$ in Theorem \ref{thm:new}, Theorem \ref{thm:3} shows that for the sub-Gaussian case, the converge rate of the estimation error is $O(\frac{1}{\sqrt{p}}+\sqrt{\frac{p}{n}})$ in the $\ell_\infty$-norm (if we omit other terms). This is due to the estimation error in Theorem \ref{thm:2}. 
Moreover, compared with the assumptions of link functions in Theorem \ref{thm:new}, there are additional assumptions in Theorem \ref{thm:3}.
\end{remark}
In order to reduce the time complexity and make the algorithm faster, we can also use the sub-sampled covariance OLS estimator. This is the same as that in the Gaussian case. 
\begin{theorem}\label{thm:new2}
Under the same assumptions as in Theorem \ref{thm:3},if we use Option 2 in Algorithm \ref{alg:1}, then with probability at least $1-4k\exp(-p)$, the outputs $\{\hat{\beta}_j^{nlr}\}_{j=1}^k$  satisfy for each $j\in [k]$
\begin{multline*}
           \|\hat{\beta}_j^{nlr}- \beta^*_j\|_\infty \leq O\big( \sqrt{\rho_2}\rho_\infty \lambda^{-\frac{1}{2}}_{\min} (\Sigma) \sqrt{\frac{p}{|S|}}\|\beta_j^*\|_\infty \\ \times \max\{1,\|\beta_j^*\|_\infty\}  +\rho_2\rho_{\infty}^2 \frac{\max\{\|\beta_j^*\|_{\infty}^2, 1\}\|\beta_j^*\|_{\infty}^2}{\sqrt{p}}\big).  
\end{multline*}
\end{theorem}
{ An unsatisfactory issue of all above
%However, to make above 
theorems is that  they need quite a few assumptions/conditions. Although almost all of them commonly appear in some related works, the assumptions on functions $\ell_j(\cdot , \beta_j^{ols})$ seem to be a little strange. 
%Theoretically, 
They are introduced for ensuring that the functions $\ell_j(\cdot, \beta_j^{ols})-1$ and $\hat{\ell}_j(\cdot, \tilde{\beta}^{ols}_j)-1$  have  roots and $\hat{c}_j$ is close to $c_j$ for large enough $n$. The following theorem shows that the sigmoid link function indeed satisfies the assumptions in 
Theorem \ref{thm:new}  when $x$ is some Gaussian and $\|\beta^{ols}_j\|_2=O(\sqrt{p})$. }
\begin{theorem}\label{thm:new3} 
Consider the case where the link function $f_j(z)= \frac{1}{1+e^{-z}}$ and $x\sim \mathcal{N}(0, \frac{1}{p}I_p)$ and $\|\beta_j^{ols}\|_2=\frac{\sqrt{p}}{20}$. Then when $\bar{c}_j=6$ and $\eta=0.22$, $\ell_j(\bar{c}_j, \beta_j^{ols})>1+\eta$. Moreover, $\ell'(z, \beta^{ols})$ is bounded by  constant $M=0.19$ on $[0, \bar{c}_j]$ from below.
\end{theorem}

As we will see later, our
experimental results show that the algorithm actually performs quite well even though some of the assumptions are not satisfied, such as the model with logistic and cubic link functions. 

	{To deal with the privacy issue, recently, \cite{wang2019estimating} studied the problem of estimating smooth Generalized Linear Model and non-linear regression problem in a model which is called \textit{Non-interactive Local Differential Privacy with public unlabeled data}. And they proposed several methods based on the Stein's lemma and zero-bias transformation. As we mentioned above, non-linear regression is one special case of our model, thus, we can easily get an algorithm which could estimate our problem in the the model of Non-interactive Local Differential Privacy with public unlabeled data.}

\section{Omitted Proofs} \label{sec:6}
\begin{proof}[{\bf Proof of Theorem 1}]
Fix $j\in [k]$, we have 
\begin{align}
    &\mathbb{E}[z_jyx] = \mathbb{E}[z_j(\sum_{i=1}^k z_if_i(\langle x, \beta_i^*\rangle )+\epsilon)x] \nonumber \\
    &= \mathbb{E}[z_j^2 f_j(\langle x, \beta_j^*\rangle)x ] = \mathbb{E}[f_j(\langle x, \beta_j^*\rangle)x], \label{aeq:1}
\end{align}
where Eq.~(\ref{aeq:1}) is due to the assumption of $\{z_j\}_{j=1}^k$ are i.i.d and $\text{Var}(z_j)=1$.  

Now, denote by $\phi(x|\Sigma)$ the multivariate normal density with mean $0$ and covariance matrix $\Sigma$. We recall the well-known property of Gaussian density $\frac{d \phi(x|\Sigma)}{d x}= -\Sigma^{-1}x\phi(x|\Sigma)$ ({\bf this is just the Stein's lemma}). Using this and the integration we have 
\begin{align}
   & \mathbb{E}[f_j(\langle x, \beta_j^*\rangle)x] = \int xf_j(\langle x, \beta_j^*\rangle) \phi(x|\Sigma)d x \nonumber \\
    &= \int -f_j(\langle x, \beta_j^*\rangle) d\phi(x|\Sigma) = \Sigma \beta_j^*\mathbb{E}[f'_j(\langle x, \beta_j^*\rangle)]. \nonumber
\end{align}
Thus we have 
\begin{equation}\label{aeq:2}
 \frac{1}{\mathbb{E}[f'_j(\langle x, \beta_j^*\rangle)]} \Sigma^{-1} \mathbb{E}[z_jyx] =  \beta_j^*.
\end{equation}
\end{proof}

\begin{proof}[{\bf Proof of Theorem 4}] 
Now we fix $j\in [k]$, by the assumptions in the model we have 
\begin{align*}
    &\mathbb{E}[z_j y x]= \mathbb{E}[z_j(\sum_{i=1}^k z_if_i(\langle x, \beta_i^*\rangle )+\epsilon)x]  \\
    &= \mathbb{E}[z_j^2f_j(\langle x, \beta_j^*\rangle) x]= \mathbb{E}[f_j(\langle x, \beta_j^*\rangle) x] \\
    & = \Sigma^{\frac{1}{2}}\mathbb{E}[vf_j(\langle v, \tilde{\beta}_j\rangle].
\end{align*}
Thus we have 
\begin{equation}\label{aeq:3}
    \beta^{ols}_j = \Sigma^{-\frac{1}{2}}\mathbb{E}[vf_j'(\langle v, \tilde{\beta}_j\rangle].  
\end{equation}
For convenience we will omit subscript $j$. Now we define the partial sum $V_{-i}=\langle v, \tilde{\beta}\rangle- v_i\tilde{\beta}_{i}$ for $i\in [p]$. We will focus on the $i$-th entry of the above expectation given in (\ref{aeq:3}). Denote the zero-bias-transformation of $v_i$ conditioned on $V_{-i}$ by $v^*_i$, we have 
\begin{align}
    &\mathbb{E}[v_if_j(\langle v, \tilde{\beta}_j\rangle]= \mathbb{E}[\mathbb{E}[v_i f_j( v_i\tilde{\beta}_{i}+V_{-i})|V_{-i}]]\nonumber \\
    & = \tilde{\beta}_i\mathbb{E}[f_j'(v^*_i\tilde{\beta}_{i}+V_{-i})]  = \tilde{\beta}_i\mathbb{E}[f_j'((v^*_i-v_i)\tilde{\beta}_{i}+\langle \tilde{\beta}, v \rangle  ) ]. \nonumber
\end{align}
Combining with the above equation, we have 
\begin{equation}\label{aeq:5}
    \beta^{ols}= \Sigma^{-\frac{1}{2}} D \tilde{\beta}=  \Sigma^{-\frac{1}{2}} D \Sigma^{\frac{1}{2}}\beta^*, 
\end{equation}
where $D$ is a diagonal matrix where the the entry $D_{ii}= \mathbb{E}[f'((v^*_i-v_i)\tilde{\beta}_{i}+\langle \tilde{\beta}, v \rangle  ) ]$. 

By the Lipschitz property of $f$ we have 
\begin{align}\label{aeq:6}
 |D_{ii}- \frac{1}{c_j}|&\leq 
    |\mathbb{E}[f'((v^*_i-v_i)\tilde{\beta}_{i}+\langle \tilde{\beta}, v \rangle  ) ]- \mathbb{E}[f'(\langle \tilde{\beta}, v \rangle  )]|\nonumber \\
    &\leq G\mathbb{E}[|(v^*_i-v_i)\tilde{\beta}_{i}|]. 
\end{align}
Now we will bound the term of $\mathbb{E}[|(v^*_i-v_i)|]$, by the same method as in \cite{erdogdu2016scaled} we can get 
\begin{equation}\label{aeq:7}
    \mathbb{E}[|(v^*_i-v_i)|] \leq \frac{3}{2}\mathbb{E}[|v_i|^3]\leq 8\kappa_x^3
\end{equation}
where the second inequality comes from $\frac{1}{\sqrt{3}}\mathbb{E} [|v_i|^3]^{\frac{1}{3}}\leq \|v\|_{\psi_2}\leq \kappa_x$.

Thus, in total we have $\max_{i\in [p]}|D_{ii}-\frac{1}{c_j}|\leq 8G\kappa_x^3\|\Sigma^{\frac{1}{2}}\beta_j^*\|_\infty$ and 
\begin{align*}
    \|\beta_j^{ols}- \frac{1}{c_j}\beta^*_j\|_\infty& = \|(\Sigma^{-\frac{1}{2}}(D-\frac{1}{c_j}I)\Sigma^{\frac{1}{2}}\|_{\infty})\|\beta^*_j\|_\infty \\
    &\leq \max_{i\in [p]}|D_{ii}-\frac{1}{c_j}|\|\Sigma^{\frac{1}{2}}\|_{\infty} \|\Sigma^{-\frac{1}{2}}\|_{\infty} \|\beta^*_j\|_\infty \\
    &\leq 8G\kappa_x^3\rho_{\infty}\|\Sigma^{\frac{1}{2}}\|_{\infty} \|\beta_j^*\|^2_{\infty}.
\end{align*}
Due to the diagonal dominance property we have 
\begin{equation*}
    \|\Sigma^{\frac{1}{2}}\|_{\infty}=\max_{i}\sum_{j=1}^p|\Sigma_{ij}^{\frac{1}{2}}|\leq 2\max \Sigma_{ii}^{\frac{1}{2}}\leq 2\|\Sigma\|_2^{\frac{1}{2}}.
\end{equation*}
Since we have $\|x\|_2\leq r$, we write 
\begin{equation*}
    r^2 \geq \mathbb{E}[\|x\|_2^2]=\text{Trace}(\Sigma)\geq p\|\Sigma\geq \frac{p\|\Sigma\|_2}{\rho_2}.
\end{equation*}
Thus we have 
   $ \|\Sigma^{\frac{1}{2}}\|_\infty\leq 2r\sqrt{\frac{\rho_2}{p}}$.
\end{proof}
Since Theorem 5 is more complex than Theorem 2, we will proof Theorem 5 first. Before that, we need to show the following lemma. \begin{lemma}\label{alemma:4}
Under Assumption 2, if we define the function $\ell_j(c, \beta)= c\mathbb{E}[f_j'(\langle x, \beta\rangle c)]$ and its empirical counter part as 
\begin{equation*}
    \hat{\ell}_j (c, \beta)= \frac{c}{n}\sum_{i=1}^n f_j'(\langle x_i, \beta\rangle c). 
\end{equation*}
Assume that each $x_i$ are i.i.d sub-Gaussian with $\|x_i\|_{\psi_2}\leq \kappa_x$ and there exist some constant $\eta, \bar{c}_j$ such that $\ell_j(\bar{c}_j, \beta_j^{ols})>1+\eta$. Then there exists $\tilde{c}_j>0$ satisfying the equation $1=\ell_j(\tilde{c}_j, \beta_j^{ols})$. Further, assume that for sufficiently large $n$ such that 
$$n\geq \Omega(\frac{p\|\Sigma\|_2\|\beta_j^{ols}\|_2^2G^2L^2\bar{c}_j^4\kappa_x^4\gamma^2 }{\eta^2}).$$
 Then, with probability at least $1-k\exp(-p)$ there exist constants $\hat{c}_j\in (0, \bar{c}_j)$ satisfying the equations 
\begin{equation*}
    1= \frac{\hat{c}_j}{n}\sum_{i=1}^n f'_j(\langle x_i, \tilde{\beta}_j^{ols}\rangle \hat{c}_j). 
\end{equation*}
Moreover, if for all $j\in [k]$ the derivative of $z\mapsto \ell_j(z, \beta_j^{ols})$ is bounded below in absolute value ({\em does not change sign}) by $M>0$ in the interval $z\in [0, \bar{c}_j]$. Then with probability at least $1-4k\exp(-p)$ for all $j\in [k]$ we have 
\begin{equation*}
    |\hat{c}_j- \tilde{c}_j|\leq  O(M^{-1}GL\bar{c}_j^2\kappa_x^2\gamma \|\Sigma\|^{1/2}_2\|\beta_j^{ols}\|_2\sqrt{\frac{p}{n}}).
\end{equation*}
\end{lemma}

\begin{proof}[{\bf Proof of Lemma \ref{alemma:4}}]
We first proof the  existence of $\tilde{c}_j$.  By the definition we know that $\ell_j(0, \beta_j^{ols})=0$ and $\ell_j(\bar{c}_j, \beta_j^{ols})>1$, since $\ell_j$ is continuous, there must exists a $\tilde{c}_j$ which satisfies $\ell_j(\tilde{c}_j, \beta_j^{ols})=1$. 

Secondly, we will proof the existence of $\hat{c}_j\in (0, \bar{c}_j)$. Denote $v_i= \Sigma^{-\frac{1}{2}}x_i$ as the whitened random variable.  Note that by Assumption 1 we know that $z_jy$ is sub-Gaussian with norm $\|z_j y\|_{\psi_2}\leq \gamma$, thus Lemma 1 holds with $\gamma$. We denote the upper bound in Lemma 1 as $\delta$, that is $\delta= C_1\kappa_x\gamma\sqrt{\frac{p}{n}}$ and suppose the event in Lemma 1 holds. Then we have 
\begin{equation*}
    \tilde{\beta}^{ols}_j \in B^\delta_{\Sigma}(\beta_j^{ols})=\{\beta: \|\Sigma^{\frac{1}{2}}(\beta-\beta_j^{ols})\|_2\leq \delta\}  . 
\end{equation*}
Thus we have for all $c\in (0, \bar{c}_j]$
\begin{align}
    &|\hat{\ell}_j(c, \tilde{\beta}_j^{ols})- \ell_j(c,  \tilde{\beta}_j^{ols})|\\&\leq \sup_{\beta\in B^\delta_{\Sigma}(\beta_j^{ols})}|\hat{\ell}_j(c, \beta)- \ell_j(c,\beta)| \nonumber \\
    &\leq \sup_{c\in (0, \bar{c}_j] }\sup_{\beta\in B^\delta_{\Sigma}(\beta_j^{ols})}|\hat{\ell}_j(c, \beta)- \ell_j(c,\beta)| \nonumber \\
    &\leq \bar{c}_j  \sup_{c\in (0, \bar{c}_j] }\sup_{\beta\in B^\delta_{\Sigma}(\beta_j^{ols})} |\frac{1}{n}\sum_{i=1}^n f_j'(\langle x_i, \beta \rangle c) - \mathbb{E}[f'_j(\langle x, \beta \rangle c)] | \nonumber \\ 
    &=\bar{c}_j  \sup_{c\in (0, \bar{c}_j] }\sup_{\beta\in B^\delta_{\Sigma}(\beta_j^{ols})} |\frac{1}{n}\sum_{i=1}^n f_j'(\langle v_i, \Sigma^{\frac{1}{2}}\beta \rangle c) - \mathbb{E}[f'_j(\langle v, \Sigma^{\frac{1}{2}}\beta \rangle c)] \nonumber \\
    &=\bar{c}_j  \sup_{c\in (0, \bar{c}_j] }\sup_{\beta\in B^\delta(\bar{\beta}_j^{ols})}|\frac{1}{n}\sum_{i=1}^n f_j'(\langle v_i, \beta\rangle c)-\mathbb{E} f_j'(\langle w, \beta\rangle c) |\nonumber \\
    &= \bar{c}_j \sup_{\beta\in B^{\bar{c}_j \delta}(\bar{\beta}_j^{ols})}|\frac{1}{n}\sum_{i=1}^nf'_j(\langle v_i , \beta\rangle )- \mathbb{E}[f'_j(\langle w, \beta \rangle)]|. \label{aeq:13}
\end{align}
Where the second equality is due to that $\beta\in  B^\delta_{\Sigma}(\beta_j^{ols})$ is equivalent to $\Sigma^{\frac{1}{2}}\beta\in B^{\delta}(\bar{\beta}_j^{ols})$. 

Thus by Lemma 2 we have with probability at least $1-\exp(-p)$ (we have $n\geq \Omega(p\kappa_x^2\gamma^2\bar{c}_j^2)$) 
\begin{align}
   &\sup_{\beta\in B^{\bar{c}_j \delta}(\bar{\beta}_j^{ols})}|\frac{1}{n}\sum_{i=1}^nf'_j(\langle v_i , \beta\rangle )- \mathbb{E}[f'_j(\langle w, \beta \rangle)]|\nonumber \\
   &\leq 2(G(\|\bar{\beta}_j^{ols}\|+\bar{c}_j\delta)\|I\|_2 + L )\sqrt{\frac{p}{n}}\nonumber \\ 
  & = O(GL\kappa_x \gamma \|\Sigma\|^{1/2}_2\|\beta_j^{ols}\|_2\sqrt{\frac{p}{n}}).  \label{aeq:14}
\end{align}
Moreover by the $G$-Lipschitz property of $f'_j$ we have for any $\beta_1, \beta_2$ and $c\in [0, \bar{c}_j]$, 
\begin{align}
    |\ell_j(c, \beta_1)- \ell_j(c, \beta_2)| &\leq G\bar{c}_j^2 \mathbb{E}[\langle v, \Sigma^{\frac{1}{2}}(\beta_1-\beta_2)\rangle ] \nonumber \\
    &\leq G\bar{c}_j^2\kappa_x\|\Sigma^{\frac{1}{2}}(\beta_1-\beta_2)\|_2\mathbb{E}[\|v\|_2] \nonumber \\
    &\leq O( G\bar{c}_j^2\kappa_x \|\Sigma^\frac{1}{2}(\beta_1-\beta_2)\|_2). \label{aeq:15}
\end{align}
Where the last inequality is due the definition of $\kappa_x$.

Take $\beta_1= \tilde{\beta}_j^{ols}$ and $\beta_2= \beta^{ols}_j$ and by Lemma 1 we have 
\begin{equation} \label{aeq:16}
    |\ell_j(c, \tilde{\beta}_j^{ols})- \ell_j(c, \beta^{ols}_j)| \leq O(G\bar{c}_j^2\kappa_x^2\gamma \sqrt{\frac{p}{n}}).
\end{equation}
Combining with (\ref{aeq:13}), (\ref{aeq:14}), (\ref{aeq:15}) and (\ref{aeq:16}), in total we have for all $c\in (0, \bar{c}_j]$
\begin{equation}\label{aeq:17}
    |\hat{\ell}_j(c, \tilde{\beta}_j^{ols})- \ell_j(c, \beta^{ols}_j)|
    \leq O(GL\bar{c}_j^2\kappa_x^2\gamma \|\Sigma\|^{1/2}_2\|\beta_j^{ols}\|_2 \sqrt{\frac{p}{n}}). 
\end{equation}
Thus, when $c=\bar{c}_j$ we have 
\begin{equation}\label{aeq:18}
    \hat{\ell}_j(\bar{c}_j, \tilde{\beta}_j^{ols})\geq 1+\eta- O(GL\bar{c}_j^2\kappa_x^2\gamma \|\Sigma\|^{1/2}_2\|\beta_j^{ols}\|_2 \sqrt{\frac{p}{n}}).
\end{equation}
Thus when $n\geq \Omega(\frac{p\|\Sigma\|_2\|\beta_j^{ols}\|_2^2G^2L^2\bar{c}_j^4\kappa_x^4\gamma^2  }{\eta^2})$. The left term of (\ref{aeq:18}) is greater than 1, also by the definition we have $\hat{\ell}_j(0, \tilde{\beta}_j^{ols})=0$. Thus, there exists a $\hat{c}_j$ such that $\hat{\ell}_j(\hat{c}_j, \tilde{\beta}_j^{ols})=1$. 

Thirdly, we will show the approximation error of $\hat{c}_j$ w.r.t $\tilde{c}_j$. Using the Taylor's series expansion of $c\mapsto \ell_j(c, \beta_j^{ols})$ around $\tilde{c}_j$ and the assumption on the derivative of $\ell$ with respect to its first argument, we obtain 
\begin{align*}
     M|\hat{c}_j- \tilde{c}_j| &\leq |\ell_j( \hat{c}_j, \beta_j^{ols})- \ell_j( \tilde{c}_j, \beta_j^{ols})| \\ 
     &= |\ell_j( \hat{c}_j, \beta_j^{ols})- \ell_j( \hat{c}_j, \tilde{\beta}_j^{ols})|+ |\ell_j( \hat{c}_j, \tilde{\beta}_j^{ols})-1|\\
     & = |\ell_j( \hat{c}_j, \beta_j^{ols})- \ell_j( \hat{c}_j, \tilde{\beta}_j^{ols})|+  |\ell_j( \hat{c}_j, \tilde{\beta}_j^{ols})-\hat{\ell}_j( \hat{c}_j, \tilde{\beta}_j^{ols})| \\
     &\leq O(GL\bar{c}_j^2\kappa_x^2\gamma \|\Sigma\|^{1/2}_2\|\beta_j^{ols}\|_2\sqrt{\frac{p}{n}}).
\end{align*}
Where the last inequality is due to (\ref{aeq:17}) and (\ref{aeq:13}). 
\end{proof}

\begin{proof}[{\bf Proof of Theorem 5 }]

 We have for each $\hat{\beta}_j^{nlr}$
\begin{align}
    &\|\hat{\beta}_j^{nlr}- \beta^*_j\|_\infty \leq \|\hat{c}_j\tilde{\beta}_j^{ols}- \tilde{c}_j \beta_j^{ols}\|_\infty + \|\tilde{c}_j \beta_j^{ols}- \beta_j^*\|_\infty\nonumber \\ 
    &\leq \|\hat{c}_j\tilde{\beta}_j^{ols}- \tilde{c}_j \beta_j^{ols}\|_\infty+ \|\tilde{c}_j \beta_j^{ols}- c_j\beta_j^{ols}\|_\infty + \|  c_j\beta_j^{ols}- \beta_j^*\|_\infty.  \label{aeq:19}
\end{align}
We first bound the term of $|\tilde{c}_j - c_j|$. By definition we have $c_j \mathbb{E}[f_j'(\langle x, \beta_j^*\rangle)]=1$ and $\tilde{c}_j\mathbb{E}[f'_j(\langle x, \beta^{ols}_j\rangle \tilde{c}_j)]=1$, we get 
\begin{align*}
|\ell_j(\tilde{c}_j, \beta_j^{ols})- \ell_j(c_j, \beta_j^{ols})|&=|1-  \ell_j(c_j, \beta_j^{ols})| \\
    &=| c_j \mathbb{E}[f_j'(\langle x, \beta_j^*\rangle)]- c_j \mathbb{E}[f_j'(\langle x, \beta_j^{ols}\rangle  c_j)]|\\
    &\leq G|c_j|\mathbb{E}[\langle x,\beta_j^*- c_j\beta_j^{ols}\rangle]\\
    & \leq Gc_j \|\beta_j^*- c_j\beta_j^{ols}\|_\infty \mathbb{E}\|x\|_1\\
    &\leq Gc_j \kappa_x \|\beta_j^*- c_j\beta_j^{ols}\|_\infty. 
\end{align*}
Thus, by the assumption of the bounded deviation of $\ell(c, \beta_j^{ols})$ on $\max\{\bar{c}_j, c_j\}$ we have 
\begin{align*}
    M|\tilde{c}_j - c_j|\leq |\ell_j(\tilde{c}_j, \beta_j^{ols})- \ell_j(c_j, \beta_j^{ols})|
    \leq Gc_j \kappa_x \|\beta_j^*- c_j\beta_j^{ols}\|_\infty.
\end{align*}
Thus by Theorem 4 we have 
\begin{equation}\label{aeq:20}
    |\tilde{c}_j - c_j|\leq O(M^{-1}rG^2c_j^2\kappa_x^4\sqrt{\rho_2}\rho_{\infty}\frac{\|\beta_j^*\|^2_\infty}{\sqrt{p}}). 
\end{equation}
Thus for the second term of (\ref{aeq:19}) we have 
\begin{equation}\label{aeq:21}
     \|\tilde{c}_j \beta_j^{ols}- c_j\beta_j^{ols}\|_\infty \leq  O(M^{-1}rG^2c_j^2\kappa_x^4\sqrt{\rho_2}\rho_{\infty}\frac{\|\beta_j^{ols}\|_\infty \|\beta_j^*\|^2_\infty}{\sqrt{p}}). 
\end{equation}
For the first term of (\ref{aeq:19}) we have 
\begin{align}
    &\|\hat{c}_j\tilde{\beta}_j^{ols}- \tilde{c}_j \beta_j^{ols}\|_\infty \leq |\hat{c}_j| \|\tilde{\beta}_j^{ols}- \beta_j^{ols}\|_\infty + |\hat{c}_j- \tilde{c}_j| \|\beta_j^{ols}\|_\infty . \nonumber  \\
    &\leq O\big( (\tilde{c}_j+ M^{-1}GL\bar{c}_j^2\kappa_x^3\gamma^2 \|\Sigma\|^{1/2}_2\|\lambda_{\min}(\Sigma)^{-\frac{1}{2}} \beta_j^{ols}\|_2{\frac{p}{n}})\nonumber \\ 
    &+ M^{-1}GL\bar{c}_j^2\kappa_x^2\gamma  \|\Sigma\|^{1/2}_2\|\beta_j^{ols}\|_2\sqrt{\frac{p}{n}} \|\beta^{ols}_j\|_\infty
    \big) \nonumber \\
    &\leq O\big( \bar{c}_j\kappa_x\gamma \lambda_{\min}^{-\frac{1}{2}}(\Sigma)\sqrt{\frac{p}{n}}\max\{1, \|\beta^{ols}_j\|_\infty \} \big).  \label{aeq:22}
\end{align}
By Theorem 4 we have the third term of (\ref{aeq:19}) is bounded by 
\begin{equation}\label{aeq:23}
    \|  c_j\beta_j^{ols}- \beta_j^*\|_\infty\leq O(c_jGr\kappa_x^3\sqrt{\rho_2}\rho_{\infty}\frac{\|\beta_j^*\|^2_\infty}{\sqrt{p}}). 
\end{equation}
By (\ref{aeq:21}), (\ref{aeq:22}) and (\ref{aeq:23}) we have 
\begin{align}\label{aeq:24}
    &\|\hat{\beta}_j^{nlr}- \beta^*_j\|_\infty \leq O\big(M^{-1}rG^2c_j^2\kappa_x^4\sqrt{\rho_2}\rho_{\infty}\frac{\|\beta_j^{ols}\|_\infty \|\beta_j^*\|^2_\infty}{\sqrt{p}} + \nonumber \\
    & \bar{c}_j\kappa_x\gamma \lambda_{\min}^{-\frac{1}{2}}(\Sigma)\sqrt{\frac{p}{n}}\max\{1, \|\beta^{ols}_j\|_\infty \}  + c_jGr\kappa_x^3\sqrt{\rho_2}\rho_{\infty}\frac{\|\beta_j^*\|^2_\infty}{\sqrt{p}} \big). 
\end{align}
Take this into the previous equation and take $G, L, M, r, \kappa_x,\gamma, C, c_j, \bar{c_j}=\Theta(1)$ we can get
\begin{multline}
    \|\hat{\beta}_j^{nlr}- \beta^*_j\|_\infty \leq O(\sqrt{\rho_2}\rho_{\infty}\frac{\|\beta_j^*\|^2_\infty\max\{1, \|\beta_j^{ols}\|_\infty\} }{\sqrt{p}}\\
    +\lambda_{\min}^{-\frac{1}{2}}(\Sigma)\sqrt{\frac{p}{n}}\max\{1, \|\beta^{ols}_j\|_\infty \} )
\end{multline}

By Theorem 4 we have 
\begin{equation*}
    \|\beta^{ols}_j\|_\infty \leq O( \|\frac{1}{c_j}\beta^*_j\|_\infty+ Gr\kappa_x^3\sqrt{\rho_2}\rho_{\infty}\frac{\|\beta_j^*\|^2_\infty}{\sqrt{p}}).
\end{equation*}
Thus we have 
\begin{multline*}
    \|\hat{\beta}_j^{nlr}- \beta^*_j\|_\infty \leq O(\max\{1, \sqrt{\rho_2}\rho_\infty\}\lambda^{-\frac{1}{2}}_{\min} (\Sigma) \sqrt{\frac{p}{n}}\|\beta_j^*\|_{\infty}\max\{1, \|\beta_j^*\|_{\infty}\}  \\ + \rho_2\rho_{\infty}^2 \frac{\max\{\|\beta_j^*\|_{\infty}^2, 1\}\|\beta_j^*\|_{\infty}^2}{\sqrt{p}}). 
\end{multline*}
Note that the equality holds when 
$$n\geq \Omega(\frac{p\|\Sigma\|_2\|\beta_j^{ols}\|_2^2G^2L^2\bar{c}_j^4\kappa_x^4\gamma^2 }{\eta^2}).$$
Which will holds when 
\begin{multline}
    n\geq \Omega(\frac{\|\Sigma\|_2p^2G^4L^2\bar{c}_j^4\kappa_x^{10}\max\{\rho_2\rho^2_{\infty},1 \}r^2\gamma^2\|\beta^*_j\|^2_{\infty}\max\{1, \|\beta^*_j\|_\infty^2\}}{\eta^2} \\
    \max\{1, \frac{1}{c_j^2}\})\geq 
    \Omega(\frac{ p^2\|\Sigma\|_2\|\beta_j^{ols}\|_\infty^2G^2L^2\bar{c}_j^4\kappa_x^4\gamma^2  }{\eta^2})
\end{multline}
\end{proof}
\begin{proof}[{\bf Proof of Theorem 2}]
Since $x$ is Gaussian, Lemma 1  holds with $\kappa_x=O(1)$. The proof of Theorem 2 is almost the same as in the proof of Theorem 4. 

 By Theorem 1 we have for each $\hat{\beta}_j^{nlr}$ 
\begin{align}
    &\|\hat{\beta}_j^{nlr}- \beta^*_j\|_2= \|\hat{c}_j\tilde{\beta}_j^{ols}- c_j\beta_j^{ols}\|_2 \\
    &\leq \hat{c}_j\|\tilde{\beta}_j^{ols}-\beta_j^{ols}\|_2+|\hat{c}_j-c_j| \|\beta_j^{ols}\|_2.  \label{aeq:35}
\end{align}
For the first term of (\ref{aeq:35}) we have 
\begin{align}
    \hat{c}_j\|\tilde{\beta}_j^{ols}-\beta_j^{ols}\|_2 &\leq (c_j+ M^{-1}GL\bar{c}_j^2\gamma^2\|\Sigma\|^{1/2}_2\|\beta_j^{ols}\|_2{\frac{p}{n}})\lambda^{-\frac{1}{2}}_{\min}(\Sigma) \nonumber \\
    & \leq O(c_j \gamma \lambda^{-\frac{1}{2}}_{\min}(\Sigma) \sqrt{\frac{p}{n}}).
\end{align}
For the second term of (\ref{aeq:35}) we have 
\begin{equation}
   |\hat{c}_j-c_j| \|\beta_j^{ols}\|_2 \leq O(M^{-1}GL\bar{c}_j^2\gamma \|\Sigma\|^{1/2}_2\|\beta_j^{ols}\|_2\sqrt{\frac{p}{n}}\frac{1}{c_j}\|\beta^*_j\|_2).
\end{equation}
In total we have 
\begin{multline}
    \|\hat{\beta}_j^{nlr}- \beta^*_j\|_2 \leq  O(c_j \gamma \lambda^{-\frac{1}{2}}_{\min}(\Sigma) \sqrt{\frac{p}{n}}\\
    +M^{-1}GL {c}_j^2\gamma \|\Sigma\|^{1/2}_2\|\beta_j^{ols}\|_2\sqrt{\frac{p}{n}}\frac{1}{c_j}\|\beta^*_j\|_2)
\end{multline}
If we assume $\eta, c_j, \gamma, C, M, G, L =\Theta(1) $ we can get 
\begin{equation}
       \|\hat{\beta}_j^{nlr}- \beta^*_j\|_2 \leq O(\max\{1, \|\beta_j^*\|_2\}\lambda^{-\frac{1}{2}}_{\min}\sqrt{\frac{p}{n}} ) .
\end{equation}
The sample show satisfies 
$$n\geq \Omega(\frac{p\|\Sigma\|_2\|\beta_j^{ols}\|_2^2G^2L^2\bar{c}_j^4\gamma^2  }{\eta^2})=\Omega(\frac{p\|\Sigma\|_2\|\beta_j^*\|_2^2G^2L^2\bar{c}_j^4\gamma^2 }{\eta^2}\frac{1}{c_j^2}) $$
\end{proof}
\begin{proof}[{\bf Proofs of Theorem 3 and 6}]
The proofs of Theorem 3 and 6 are almost the same as them of Theorem 2 and 5, respectively. Instead of using Lemma 1 here we will use Lemma 3. For convenience we omit them here. 
\end{proof}
\begin{proof}[{\bf Proof of Theorem 7 }] The proof is motivated by \cite{erdogdu2019scalable}. 
For the function $f(z)=\frac{1}{1+e^{-z}}$ we have 
\begin{equation*}
    f'(z) = \frac{e^z}{(1+e^z)^2}, f^{(3)}(z) = \frac{e^z(1-4e^z+e^{2z})}{(1+e^z)^4}. 
\end{equation*}
It is notable that both of $f'(z)$ and $f^{(3)}$ are even functions. We will use the local convexity for $z\geq 0$ around $z=2.5$, thus we can get 
\begin{equation}
    f'(z)\geq f'(2.5)+f^{(2)}(2.5)(z-2.5)= a-bz, 
\end{equation}
where $a=f'(2.5)- 2.5f^{(2)}(2.5)$ and $b=-f^{(2)}(2.5)$. Denote $W=\mathcal{N}(0, 1)$ and $\phi$ as the density of the standard Gaussian distribution, by the assumption we have 
\begin{align*}
    \ell(z, \beta^{ols})&=z\mathbb{E}[f'(\langle x, \beta^{ols}\rangle )z]= z\mathbb{E}[f'(\frac{Wz}{20})]\\
    &= 2z\int_{0}^{\infty} f'(\frac{wz}{20})\phi(w)dw \\
    &\geq 2z \int_{0}^{\frac{20a}{bz}}(a-\frac{bwz}{20})\phi(w)dw \\
    &= 2z\{a\Phi(\frac{20a}{bz})-\frac{a}{2}-\frac{bz}{20\sqrt{2\pi}}(1-\exp(-\frac{200a^2}{b^2z^2})\}.
\end{align*}
Take $z=6$ and $\eta=0.22$ we have $\ell(z, \beta^{ols})>1+\eta$.
Next we will prove that $ \ell'(z, \beta^{ols})$ is lower bounded by 0.19. 

By the Stein's lemma (the special case of Definition 3 with $z$ is Gaussian) we can get 
\begin{align*}
    \ell'(z, \beta^{ols})&=\mathbb{E}[f'(\frac{Wz}{20})]+(\frac{z}{20})^2\mathbb{E}[f^{(3)}(\frac{Wz}{20})]\\
    &\geq 2\{a\Phi(\frac{20a}{bz})-\frac{a}{2}-\frac{bz}{20\sqrt{2\pi}}(1-\exp(-\frac{200a^2}{b^2z^2})\} \\
    &-\frac{9}{100}\max|f^{(3)}|\geq 0.19. 
\end{align*}
\end{proof}

\section{Experiments}\label{sec:7}
We conduct experiments on three types of link functions: polynomial, sigmoid, and logistic function, as well as an arbitrary mix of them. Formally, the polynomial link functions include $f(x)=x, x^3, x^5$; the sigmoid link function is defined as $f(x)=(1+e^{-x})^{-1}$; the logistic link function refers to $f(x)=\log(1+e^{-x})$. Due to the statistical setting we focused on in the paper, we will perform our algorithm on the synthetic data, and the same experimental setting has been used in the previous work such as \cite{na2018high,yang2017learning,erdogdu2016scaled}.

\noindent\textbf{Experimental setting.} We sample all coefficient $z_{i,j}$ and noise $\epsilon$ i.i.d. from standard Gaussian distribution $N(0,1)$ across each experiment. Each $\beta^*_j$ is generated by sampling from $N(1, 16\mathbb{I}_d)$. We consider two distributions for generating $x$: Gaussian and Uniform distribution (corresponds to thr sub-Gaussian case). To satisfy the requirement of Theorem~\ref{thm:new3}, in both cases the standard variance is scaled by $1/p$ and this is also used in \cite{erdogdu2016scaled}, where $p$ is the data dimension. Thus, in the Gaussian case, $x\sim N(0,\frac{1}{p}\mathbb{I}_p)$, while in the sub-Gaussian case $x$ is sampled from a uniform distribution, {\em i.e., $x~\sim U([-1/p, 1/p])^p$}. Finally, given the list $F=[f_1,\ldots,f_k]$ of link functions, the response $y$ is computed via \eqref{eq:1}.  It is notable that the experimental results with different number of link functions $k$ are incomparable since when $k$ is changed Model (\ref{def:1}) will also be changed.

These experiments are divided into two parts, examining how the sample size $n$ and the size of the sub-sample set $S$ affect the algorithm performance. In the first part we vary $n$ from $100\,000$ to $500\,000$ with fixed $p=20$ and $|S|=n$, while in the second part we vary $|S|$ from $0.01n$ to $n$, with fixed $n=500\,000$ and $p=20$. In each part we test the algorithm against various data distribution/link function combinations. For each experiment, in order to support our theoretical analysis, we will use the (maximal) relative error as the measurement, that is when $x$ is Gaussian we use $\max_{i\in [k]}\frac{\|\beta_i-\beta_i^*\|_2}{\|\beta_i^*\|_2}$ and when $x$ is sub-Gaussian we will use $\max_{i\in [k]}\frac{\|\beta_i-\beta_i^*\|_\infty}{\|\beta_i^*\|_\infty}$. For each experiment we repeat 20 times and take the average as the final output.

\textbf{Experiment results. }Each of Figure~\ref{fig:monomial-link}-\ref{fig:logexp-link} illustrates the result for a single type of link function. We can see that the relative error decreases steadily as the sample size $n$ grows which is due to the $O(\frac{1}{\sqrt{n}})$ converge rate as our theorem states. Also, we can see that the size of $S$ doesn't affect the final relative error much if $|S|$ is large enough, {\em i.e., in all cases, choosing large enough $|S|\geq 0.2n$ is sufficient to achieve a relative error roughly the same as when $|S|=n$,} which also has been mentioned in our theorems theoretically.   

We further investigate the case when $F$ contains different types of link functions. In Figure~\ref{fig:polynomial-sub-gaussian}, we let $F$ contain polynomials with different degrees ($x, x^3, x^5$), and there are roughly $\frac{k}{3}$ functions for each degree. Similarly, in Figure~\ref{fig:mixed-sub-gaussian} we also mix polynomial links with the other two types of links, i.e., logistic link and log-exponential link, and there are roughly $\frac{k}{3}$ functions for each type of link function. Our algorithm achieves similar performance as in the previous settings. 
%Thus we can see that our algorithm is also effective for the case with different type of link functions in model (\ref{eq:1}). 
%\vspace{-0.1in}

\begin{figure}[!htbp]
    \centering
    \begin{subfigure}[b]{0.5\textwidth}
    \includegraphics[width=\textwidth]{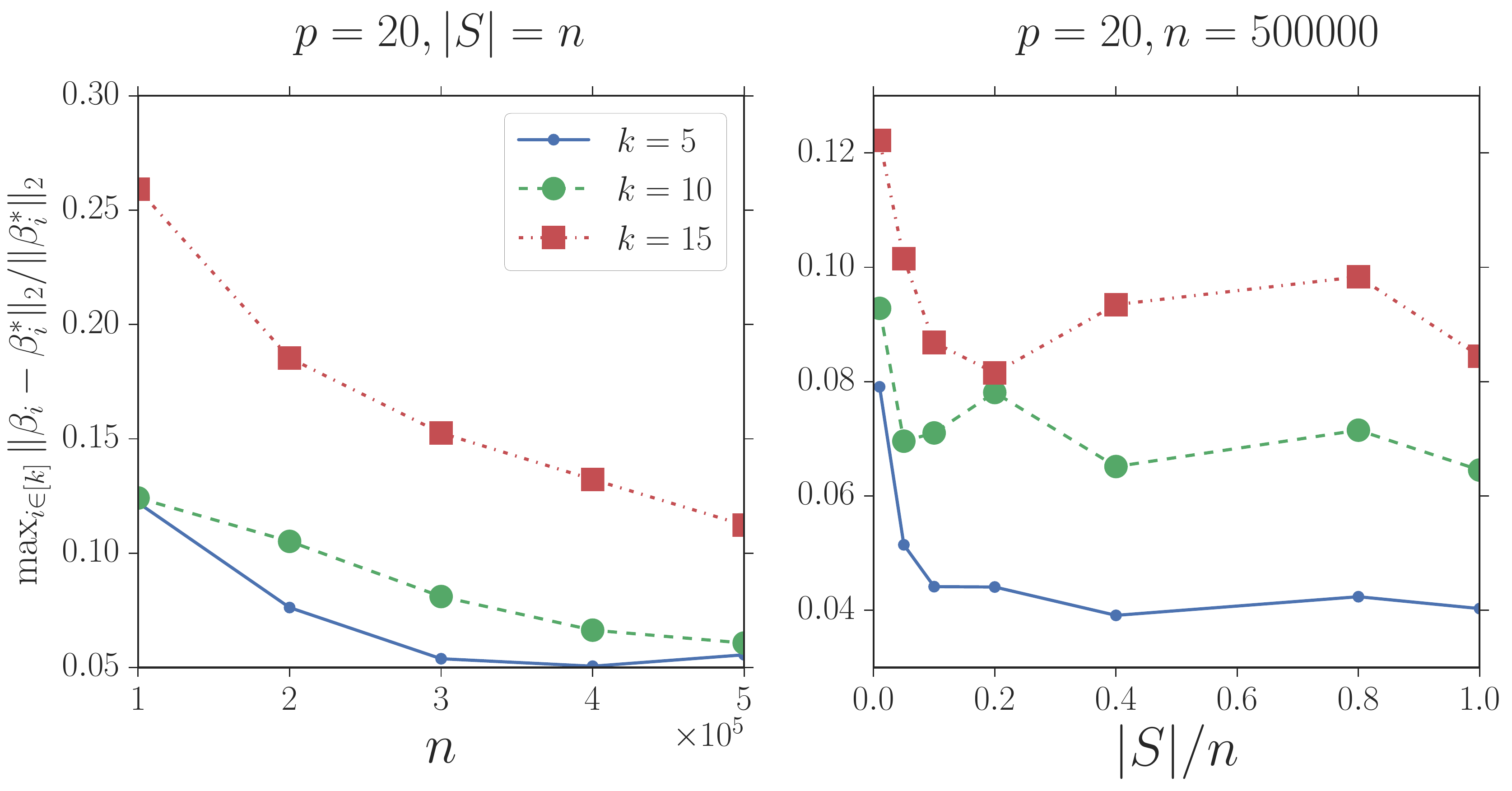}
    \caption{Gaussian data\label{fig:monomial-gaussian}}
    \end{subfigure}
    ~
    \begin{subfigure}[b]{0.5\textwidth}
    \includegraphics[width=\textwidth]{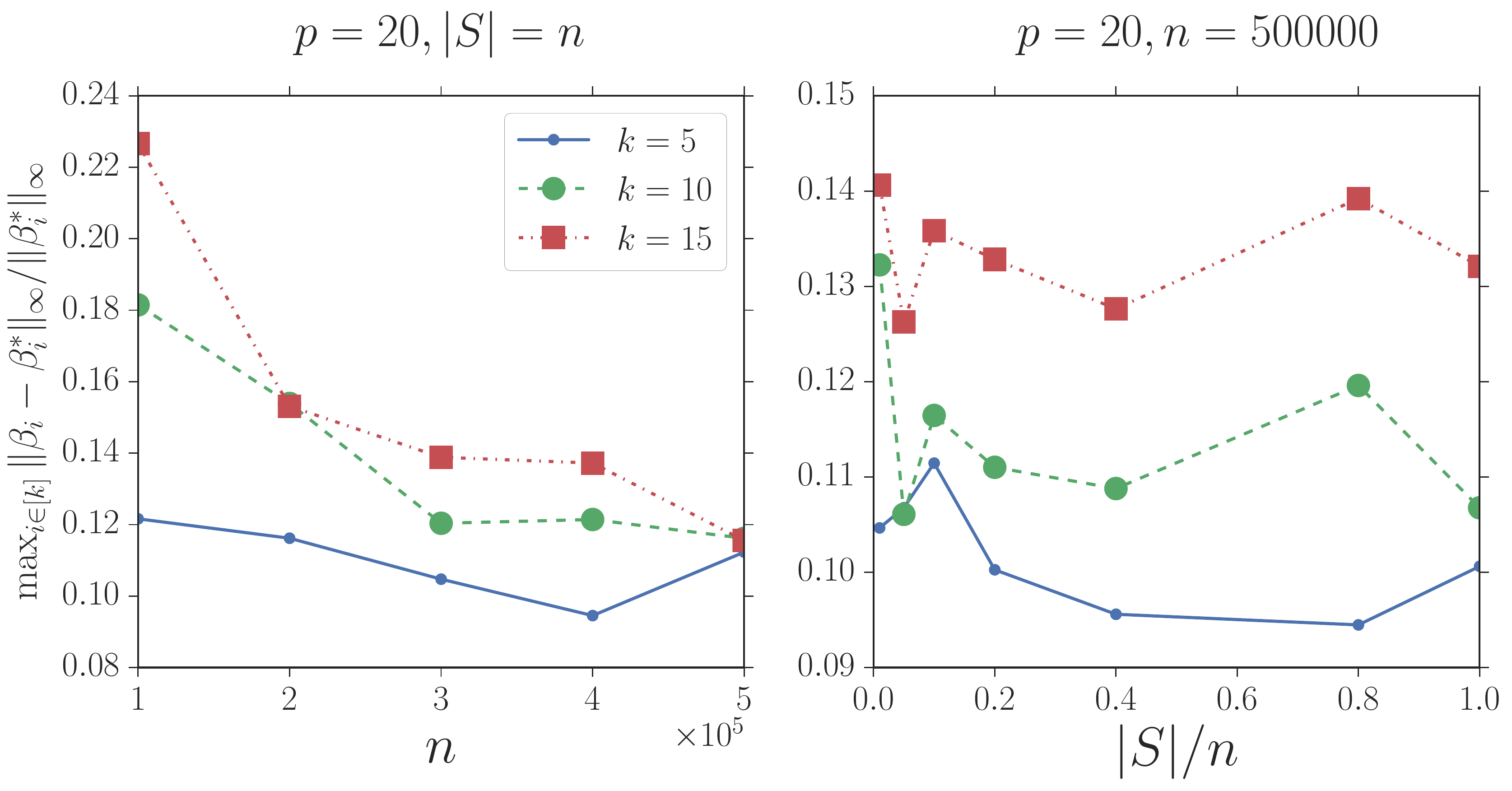}
    \caption{sub-Gaussian data\label{fig:monomial-sub-gaussian}}
    \end{subfigure}    
    \caption{Single type of link function $f(x)=x^3$\label{fig:monomial-link}}
\end{figure}

\begin{figure}[!htbp]
    \centering
    \begin{subfigure}[b]{0.5\textwidth}
    \includegraphics[width=\textwidth]{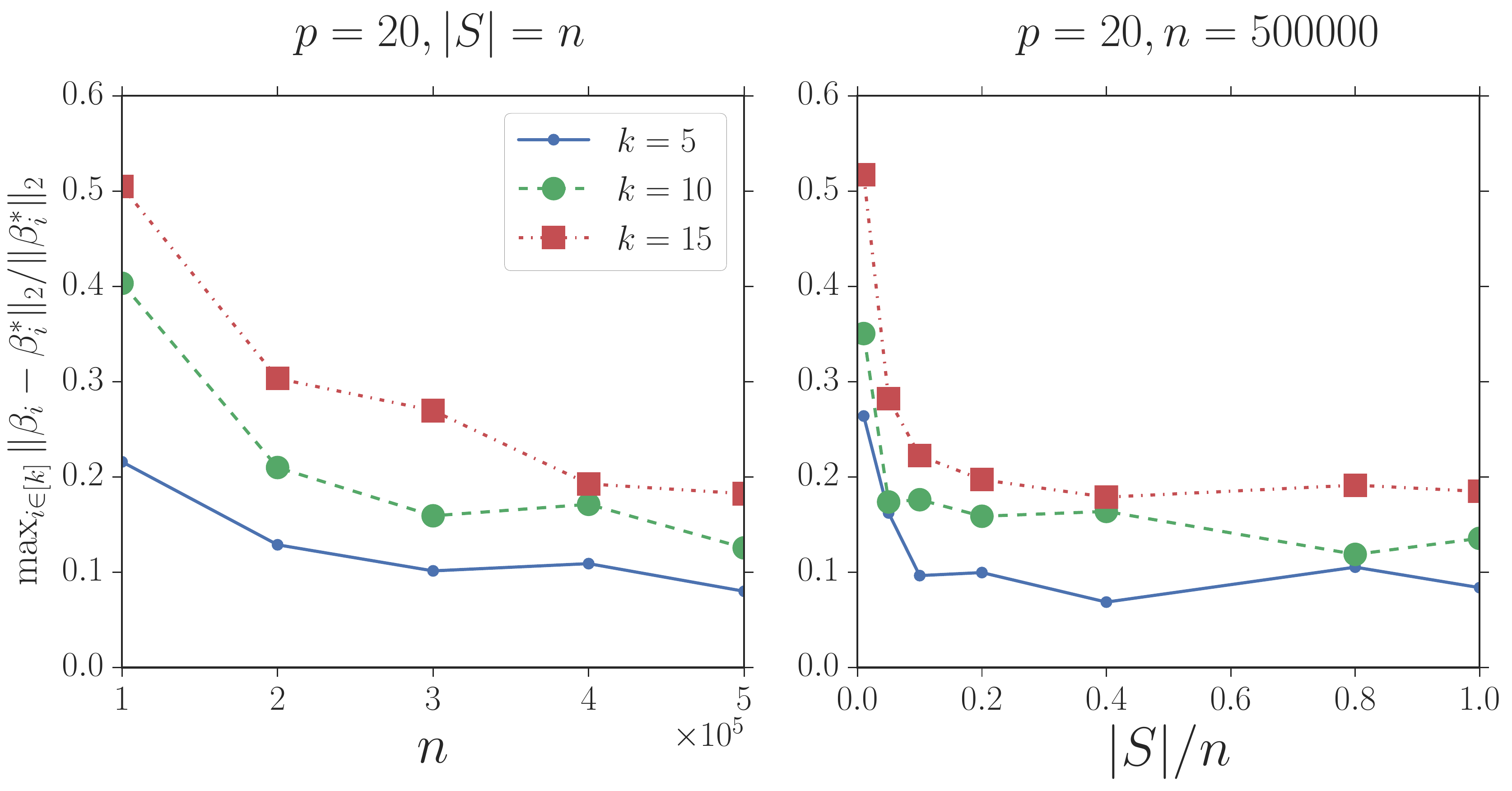}
    \caption{Gaussian data\label{fig:logistic-gaussian}}
    \end{subfigure}
    ~
    \begin{subfigure}[b]{0.5\textwidth}
    \includegraphics[width=\textwidth]{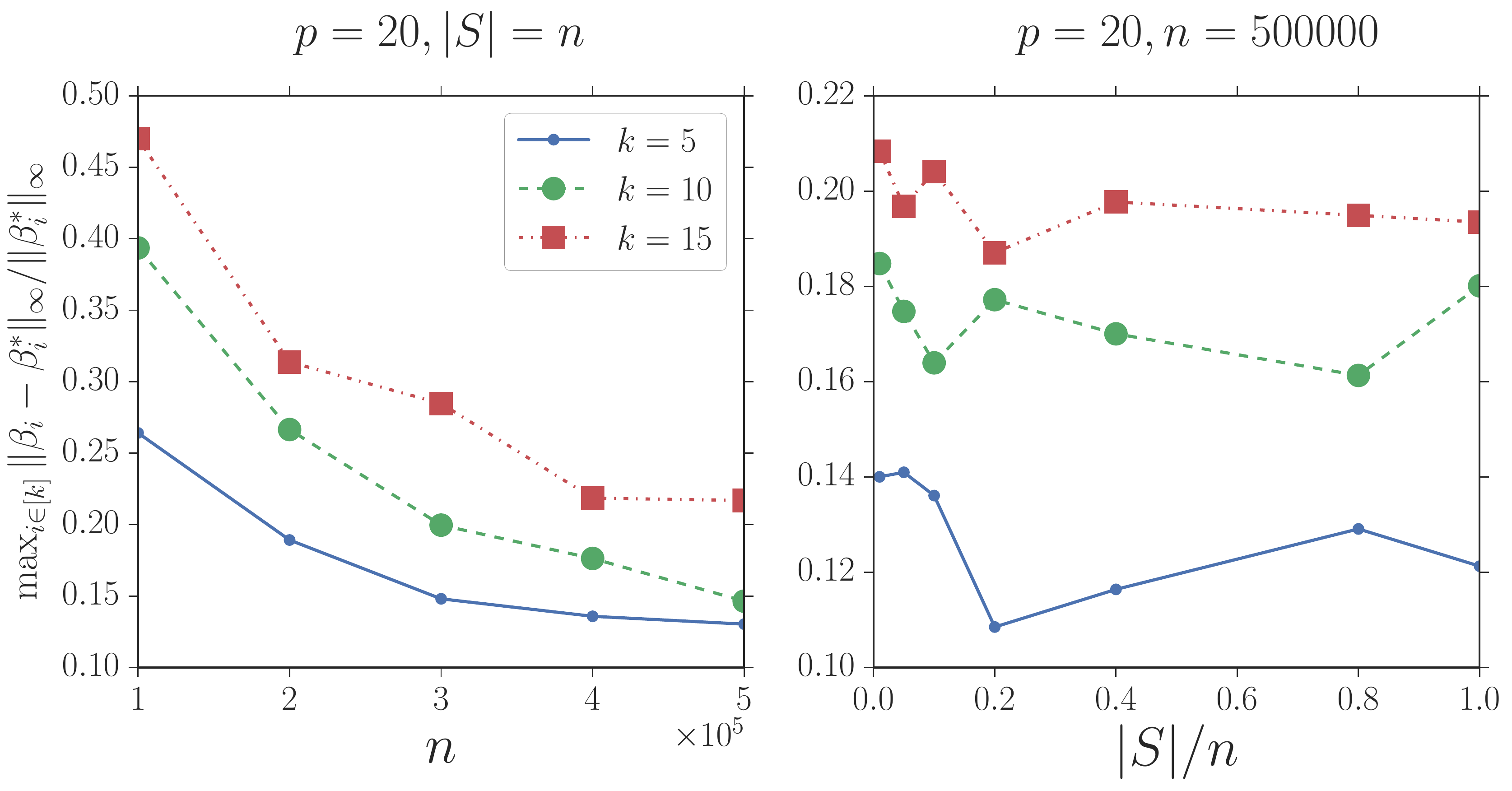}
    \caption{sub-Gaussian data\label{fig:logistic-sub-gaussian}}
    \end{subfigure}    
    \caption{Single type of link function $f(x)=(1+e^{-x})^{-1}$\label{fig:logistic-link}}
\end{figure}

\begin{figure}[!htbp]
    \centering
    \begin{subfigure}[b]{0.5\textwidth}
    \includegraphics[width=\textwidth]{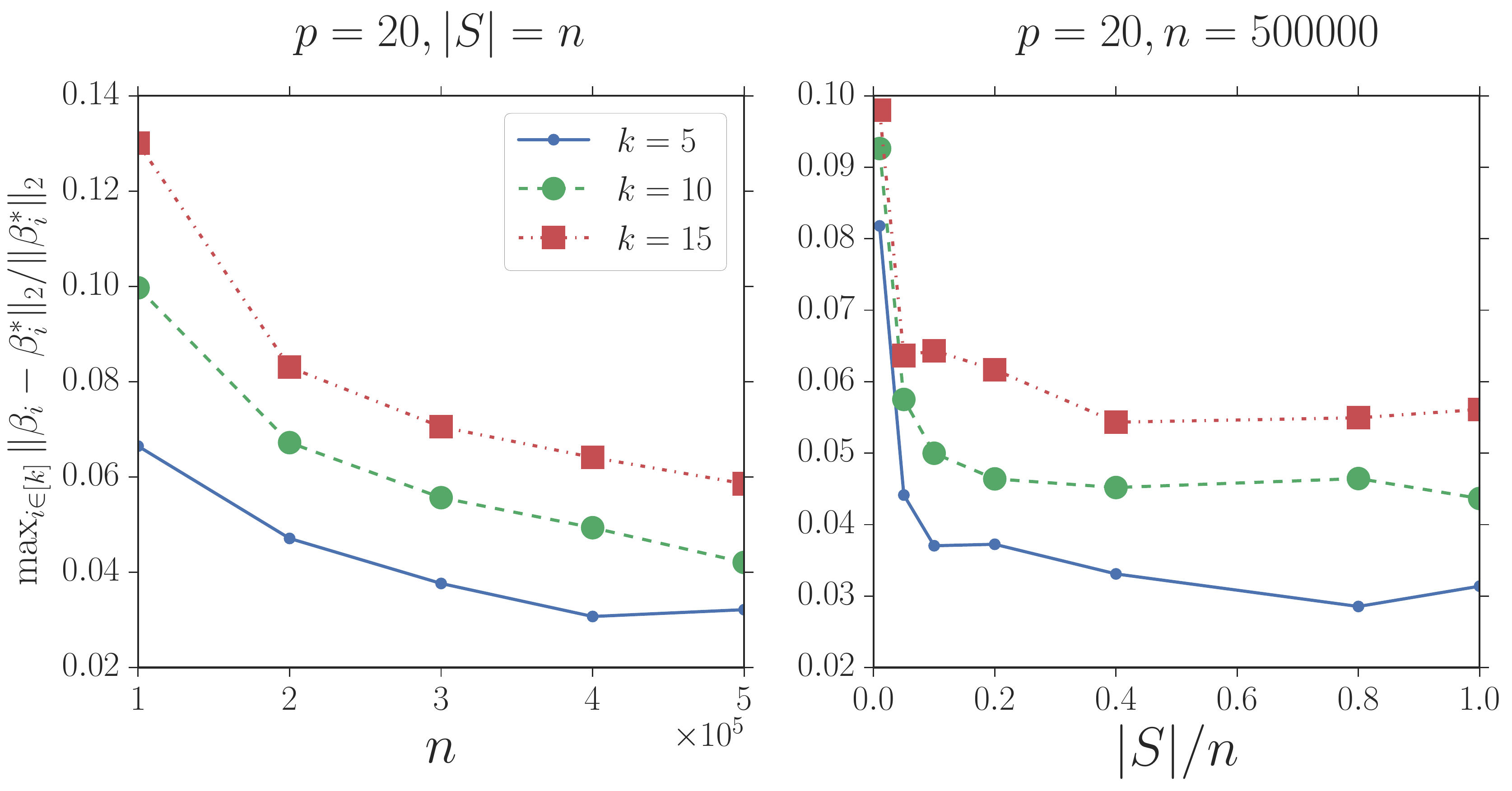}
    \caption{Gaussian data\label{fig:logexp-gaussian}}
    \end{subfigure}
    ~
    \begin{subfigure}[b]{0.5\textwidth}
    \includegraphics[width=\textwidth]{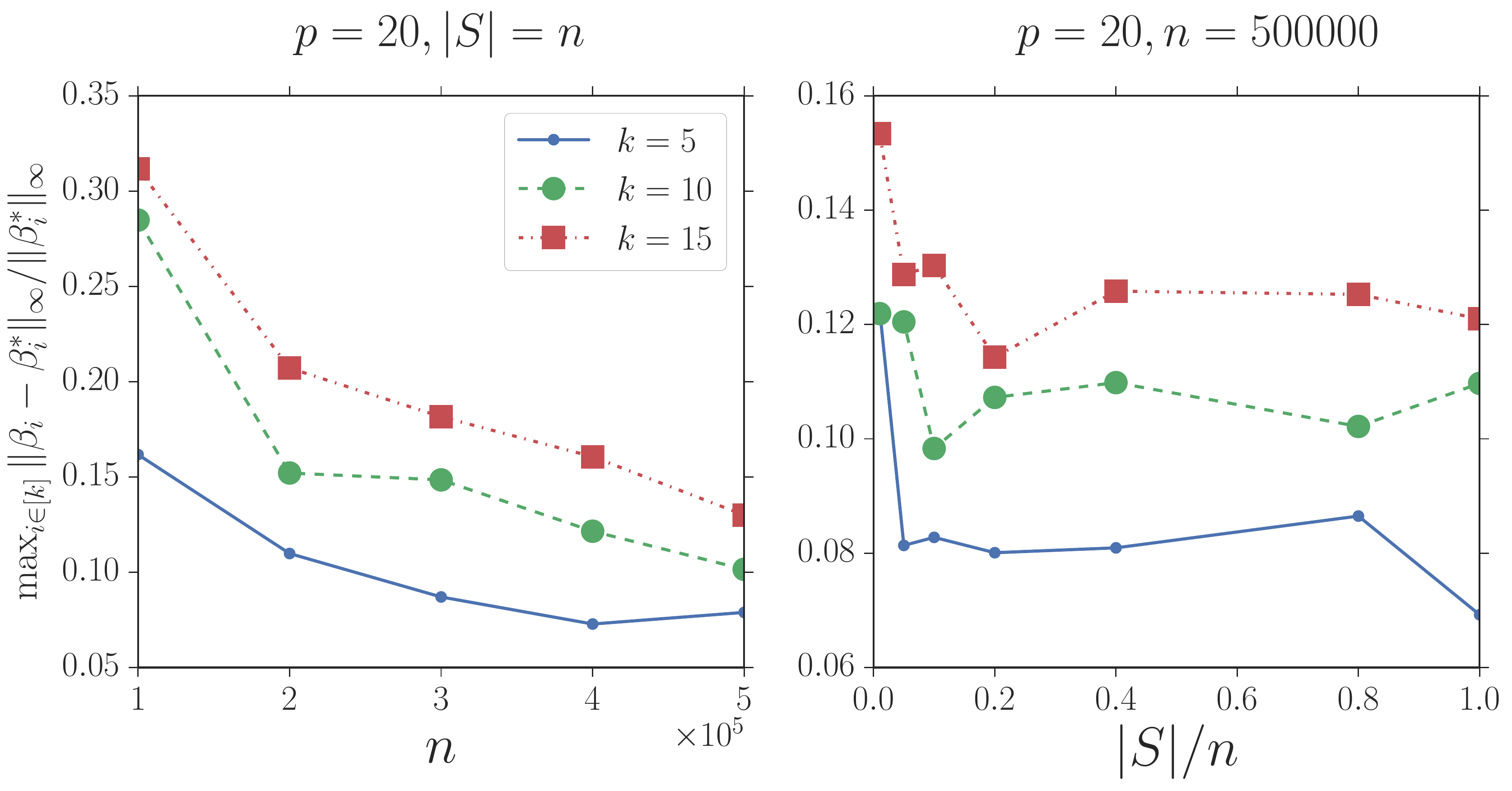}
    \caption{sub-Gaussian data\label{fig:logexp-sub-gaussian}}
    \end{subfigure}    
    \caption{Single type of link function $f(x)=\log(1+e^{-x})$\label{fig:logexp-link}}
\end{figure}

\begin{figure}[!htbp]
    \centering
    \begin{subfigure}[b]{0.5\textwidth}
    \includegraphics[width=\textwidth]{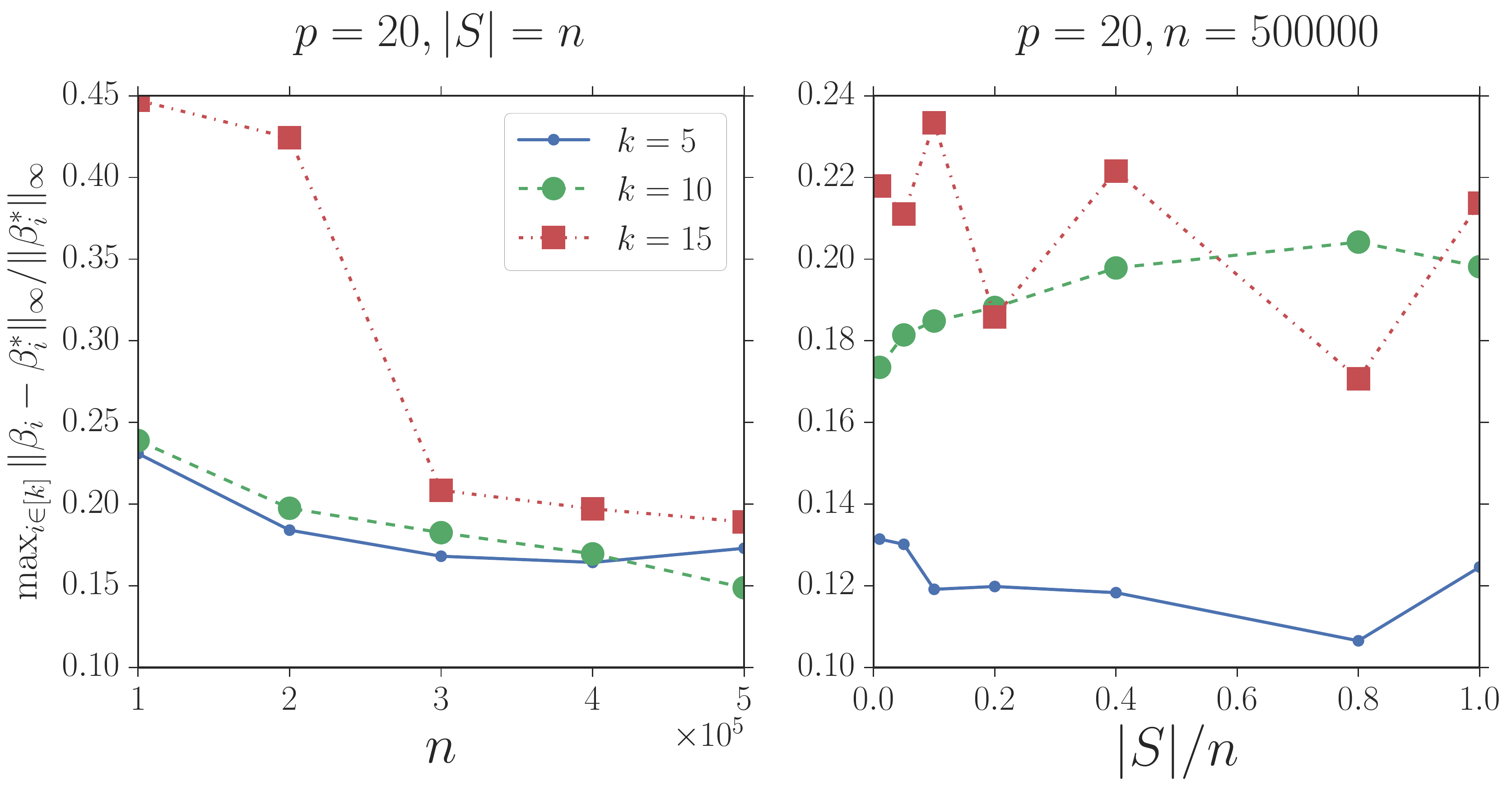}
    \caption{$f(x)=x, x^3, x^5$, sub-Gaussian data\label{fig:polynomial-sub-gaussian}}
    \end{subfigure}
    ~
    \begin{subfigure}[b]{0.5\textwidth}
    \includegraphics[width=\textwidth]{figs/results_for_monomial-sub-gaussian_err_vs_n_s-0831.pdf}
    \caption{$f(x)=x^3, (1+e^{-x})^{-1}, \log(1+e^{-x})$, sub-Gaussian data\label{fig:mixed-sub-gaussian}}
    \end{subfigure}    
    \caption{Mixed different type of link functions\label{fig:mixed-link}}
\end{figure}

\section{Conclusion}\label{sec:8}
{We studied a new model called {\it stochastic linear combination of non-linear regressions} and provided the first estimation error bounds for both Gaussian and bounded sub-Gaussian cases. Our algorithm is based on Stein's lemma and its generalization, the zero-bias transformation. Moreover, we used the sub-sampling of the covariance matrix to accelerate our algorithm. Finally, we conducted experiments whose results support our theoretical analysis.}

{There are still many open problems: 1) In the paper, we only focused on the low dimensional case (where $n\gg p$). However, it is unknown whether we can extend our methods to the high dimensional sparse case. 2) In the paper we assume $x$ is either Gaussian or Sub-Gaussian, we do not know its behaviors in the Sub-Exponential distribution case. 3) Where we can further accelerate our methods is still unknown. We will leave these problems as future work. }
\section*{References}
\bibliographystyle{elsarticle-num} 
\bibliography{nonlinear}
\end{document}